%% file: main.tex
\documentclass[nohyperref]{article}

\usepackage{microtype}
\usepackage{graphicx}
\usepackage{subfigure}
\usepackage{booktabs} % for professional tables

\usepackage{hyperref}

\usepackage[accepted]{icml2022}

\usepackage{amsmath}
\usepackage{amssymb}
\usepackage{mathtools}
\usepackage{amsthm}

\usepackage[capitalize,noabbrev]{cleveref}

\usepackage{url}
\usepackage{amsfonts}
\usepackage{multirow}
\usepackage{multicol}
\usepackage{float}
\usepackage{stfloats}
\usepackage{extarrows}
\usepackage{caption}
\usepackage{thmtools}
\usepackage{thm-restate}
\usepackage{IEEEtrantools}

\theoremstyle{plain}
\newtheorem{theorem}{Theorem}[section]
\newtheorem{proposition}[theorem]{Proposition}

\theoremstyle{definition}

\theoremstyle{remark}

\usepackage[textsize=tiny]{todonotes}

\icmltitlerunning{Conditional GANs with Auxiliary Discriminative Classifier}

\begin{document}

\twocolumn[
\icmltitle{Conditional GANs with Auxiliary Discriminative Classifier}



\icmlsetsymbol{equal}{*}

\begin{icmlauthorlist}
\icmlauthor{Liang Hou}{disrc,ucas}
\icmlauthor{Qi Cao}{disrc}
\icmlauthor{Huawei Shen}{disrc,ucas}
\icmlauthor{Siyuan Pan}{sjtu}
\icmlauthor{Xiaoshuang Li}{sjtu}
\icmlauthor{Xueqi Cheng}{lndst,ucas}
\end{icmlauthorlist}

\icmlaffiliation{disrc}{Data Intelligence System Research Center, Institute of Computing Technology, Chinese Academy of Sciences, Beijing, China}
\icmlaffiliation{ucas}{University of Chinese Academy of Sciences, Beijing, China}
\icmlaffiliation{sjtu}{Shanghai Jiao Tong University, Shanghai, China}
\icmlaffiliation{lndst}{CAS Key Laboratory of Network Data Science and Technology, Institute of Computing Technology, Chinese Academy of Sciences, Beijing, China}

% \icmlcorrespondingauthor{Liang Hou}{houliang17z@ict.ac.cn}
\icmlcorrespondingauthor{Huawei Shen}{shenhuawei@ict.ac.cn}

\icmlkeywords{Conditional Generative Adversarial Networks, Discriminative Classifier, Machine Learning, ICML}

\vskip 0.3in
]



\printAffiliationsAndNotice{}  % leave blank if no need to mention equal contribution

\begin{abstract}
Conditional generative models aim to learn the underlying joint distribution of data and labels to achieve conditional data generation.
Among them, the auxiliary classifier generative adversarial network (AC-GAN) has been widely used, but suffers from the problem of low intra-class diversity of the generated samples.
The fundamental reason pointed out in this paper is that the classifier of AC-GAN is generator-agnostic, which therefore cannot provide informative guidance for the generator to approach the joint distribution, resulting in a minimization of the conditional entropy that decreases the intra-class diversity.
Motivated by this understanding, we propose a novel conditional GAN with an auxiliary discriminative classifier (ADC-GAN) to resolve the above problem.
Specifically, the proposed auxiliary discriminative classifier becomes generator-aware by recognizing the class-labels of the real data and the generated data discriminatively.
Our theoretical analysis reveals that the generator can faithfully learn the joint distribution even without the original discriminator, making the proposed ADC-GAN robust to the value of the coefficient hyperparameter and the selection of the GAN loss, and stable during training.
Extensive experimental results on synthetic and real-world datasets demonstrate the superiority of ADC-GAN in conditional generative modeling compared to state-of-the-art classifier-based and projection-based conditional GANs.
\end{abstract}

\section{Introduction}

Generative adversarial networks (GANs)~\cite{NIPS2014_5ca3e9b1} have achieved substantial progress in learning high-dimensional, complex data distribution such as images~\cite{brock2018large,Karras_2019_CVPR,Karras_2020_CVPR,NEURIPS2020_8d30aa96,karras2021aliasfree}.
Standard GANs consist of a generator network, which transfers latent codes sampled from tractable distributions such as Gaussian in the latent space to data points in the data space, and a discriminator network, which attempts to distinguish real data and generated data.
The generator is trained in an adversarial game against the discriminator so that it can learn the data distribution at the Nash equilibrium.
Remarkably, training GANs unconditionally is difficult to achieve equilibrium, making the generator prone to mode collapse~\cite{NIPS2016_8a3363ab,NEURIPS2018_288cc0ff,Chen_2019_CVPR}.
In addition, practitioners are interested in being able to control in advance the content of the generated samples~\cite{yan2015attribute2image,tan2020michigan} in practical applications.
A promising solution to these issues is conditioning the generator, leading to conditional GANs.

Conditional GANs (cGANs)~\cite{mirza2014conditional} is a family of variants of GANs that leverages the side information from annotated labels of samples to implement and train a conditional generator for conditional image generation from class-labels~\cite{pmlr-v70-odena17a,miyato2018cgans,brock2018large}.
To implement the conditional generator, the common technique nowadays injects the conditional information via conditional batch normalization~\cite{NIPS2017_6fab6e3a,hou2021slimmable}.
To train the conditional generator, a lot of effort put into effectively injecting the conditional information into the discriminator or auxiliary classifier that guides the conditional generator~\cite{odena2016semi,miyato2018cgans,zhou2018activation,Kavalerov_2021_WACV,NEURIPS2020_f490c742,zhou2020omni}.
Among them, the auxiliary classifier generative adversarial network (AC-GAN)~\cite{pmlr-v70-odena17a} has been widely used due to its simplicity and extensibility.
Specifically, AC-GAN utilizes an auxiliary classifier that first attempts to recognize the labels of data and then teaches the generator to produce label-consistent (classifiable) data.
However, it has been reported that AC-GAN suffers from the low intra-class diversity problem in the generated samples, especially on datasets with a large number of classes~\cite{pmlr-v70-odena17a,shu2017ac,NEURIPS2019_4ea06fbc}.

In this study, we point out that the fundamental reason for the low intra-class diversity problem of AC-GAN is that the classifier is agnostic to the generated data distribution and thus cannot provide informative guidance for the generator to learn the target distribution.
Motivated by this understanding, we propose a novel conditional GAN with an auxiliary discriminative classifier, namely ADC-GAN, to resolve the above problem by enabling the classifier to be aware of the generated data distribution as well as the real data distribution.
To this end, the discriminative classifier is trained to distinguish between the real and generated data while recognizing their class-labels.
The discriminative capability allows the classifier to provide the discrepancy between the real and generated data distributions like the discriminator, and the classification capability enables it to capture the dependencies between data and labels.
We show in theory that the generator of our proposed ADC-GAN can learn the joint data and label distribution under the optimal discriminative classifier even without the discriminator, making the method robust to the value of the coefficient hyperparameter and the selection of the GAN loss and stable during training.
We also highlight the superiority of ADC-GAN compared to the two most related works (TAC-GAN~\cite{NEURIPS2019_4ea06fbc} and PD-GAN~\cite{miyato2018cgans}) by analyzing their potential issues and limitations.
Results on synthetic data clearly show that the proposed ADC-GAN successfully resolves the problem of AC-GAN by faithfully recovering the joint distribution of real data and labels.
Extensive experiments based on two popular codebases demonstrate the effectiveness of the proposed ADC-GAN compared with state-of-the-art cGANs in conditional generative modeling.

\section{Preliminaries and Analysis}

\subsection{Generative Adversarial Networks}

Generative adversarial networks (GANs)~\cite{NIPS2014_5ca3e9b1} consist of two types of neural networks: the generator $G:\mathcal{Z}\to\mathcal{X}$ that maps a latent code $z\in\mathcal{Z}$ endowed with an easily sampled distribution $P_Z$ to a data point $x\in\mathcal{X}$, and the discriminator $D:\mathcal{X}\to[0,1]$ that distinguishes between real data that sampled from the real data distribution $P_X$ and fake data that sampled from the generated data distribution $Q_X=G_{\sharp}P_Z$ induced by the generator.
The goal of the generator is to confuse the discriminator by producing data that are as real as possible.
Formally, the objective functions for the discriminator and generator are defined as follows:
\begin{IEEEeqnarray}{rCl}\label{eq:gan}
\min_G\max_D V(G,D) & = & \mathbb{E}_{x\sim P_X}[\log D(x)] \nonumber \\ & + & \mathbb{E}_{x\sim Q_X}[\log (1-D(x))].
\end{IEEEeqnarray}
Theoretically, learning the generator under the optimal discriminator can be regarded as minimizing the Jensen-Shannon (JS) divergence between the real data distribution and the generated data distribution, i.e.,~$\min_G\mathrm{JS}(P_X\|Q_X)$.
This would enable the generator to restore the real data distribution at its optimum.
However, the training of GANs on complex natural images is typically unstable~\cite{che2016mode}, especially in the absence of supervision such as conditional information.
In addtition, the content of the images generated by GANs cannot be specified in advance.

\subsection{Base Method: AC-GAN}

Learning GANs with conditional information can not only improve the training stability but also achieve conditional generation.
As one of the most representative conditional GANs, AC-GAN~\cite{pmlr-v70-odena17a} utilizes an auxiliary classifier $C:\mathcal{X}\to\mathcal{Y}$ to learn the dependencies between data and labels endowed with a label prior $P_Y$ and then encourages the conditional generator $G:\mathcal{Z}\times\mathcal{Y}\to\mathcal{X}$ to generate as much classifiable data as possible.
The objective functions for the discriminator, the auxiliary classifier, and the generator of AC-GAN\footnote{We follow the common practice in the literature to adopt the stable version instead of the original one. We also provide an analysis of the original AC-GAN in \cref{sec:acgan_full}.} are defined as follows:
\begin{IEEEeqnarray}{rCl}\label{eq:acgan}
\max_{D,C} V(G,D) & + & \lambda\cdot\left(\mathbb{E}_{x,y\sim P_{X,Y}}[\log C(y|x)]\right), \\
\min_{G} V(G,D) & - & \lambda\cdot\left(\mathbb{E}_{x,y\sim Q_{X,Y}}[\log C(y|x)]\right),
\end{IEEEeqnarray}
where $\lambda>0$ is a coefficient hyperparameter, $P_{X,Y}$ indicates the joint distribution of real data and labels, and $Q_{X,Y}=G_{\sharp}(P_{Z}\times P_{Y})$ denotes the joint distribution of the generated data and labels induced by the conditional generator.

\begin{restatable}{proposition}{acganC}\label{pro:acgan-c}
% \begin{proposition}
For fixed generator, the optimal classifier of AC-GAN has the form of $C^*(y|x)=\frac{p(x,y)}{p(x)}$.
% \end{proposition}
\end{restatable}

\begin{restatable}{theorem}{acganG}\label{thm:acgan-g}
% \begin{theorem}\label{thm:acgan-g}
Given the optimal classifier, at the equilibrium point, optimizing the classification task for the generator of AC-GAN is equivalent to:
\begin{equation}\label{eq:acgan-g}
\min_G \mathrm{KL}(Q_{X,Y}\|P_{X,Y})-\mathrm{KL}(Q_{X}\|P_{X})+H_Q(Y|X),
\end{equation}
where $H_Q(Y|X)=-\int \sum_{y} q(x,y)\log q(y|x)\mathrm{d}x$ is the conditional entropy of the generated samples.
% \end{theorem}
\end{restatable}

The proofs of all theorems are referred to \cref{sec:proofs}.
Our \cref{thm:acgan-g} exposes two shortcomings of AC-GAN.
Firstly, maximization of the KL divergence between the marginal generator and data distributions ($\max_G \mathrm{KL}(Q_X\|P_X)$) contradicts the goal of conditional generative modeling that matches $Q_{X,Y}$ with $P_{X,Y}$.
Although this issue can be mitigated to some extent by the adversarial game between the discriminator and generator that minimizes the JS divergence between the two marginal distributions ($\min_G \mathrm{JS}(Q_X\|P_X)$), we find that it still has a negative impact on training stability and generation performance.
Secondly, minimization of the entropy of labels conditioned on data of the generated distribution ($\min_G H_Q(Y|X)$) will result in the label of the generated data being deterministic.
In other words, it forces the generated data for each class away from the classification hyperplane, explaining the low intra-class diversity of the generated samples in AC-GAN, especially when the distributions of different classes have non-negligible overlap, which occurs naturally as the fact that neither state-of-the-art classifiers nor human beings can achieve $100\%$ classification accuracy on real-world datasets~\cite{russakovsky2015imagenet}.
The original AC-GAN, whose classifier is trained from both real and generated samples, suffers from the same issue (cf. \cref{sec:acgan_full}).

\section{Proposed Method: ADC-GAN}

\begin{figure*}[tbp]
\begin{center}
\subfigure[PD-GAN]{
\label{fig:pdgan}
\includegraphics[height=0.21\textheight]{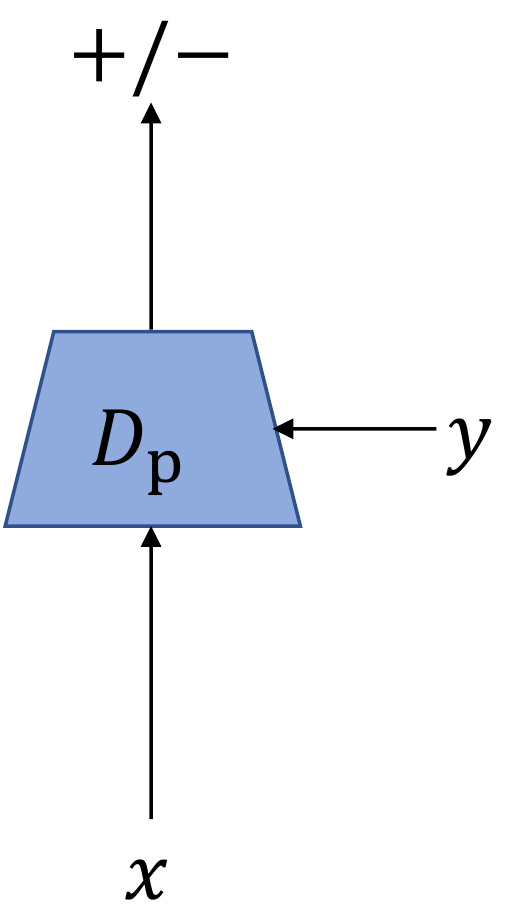}}
\subfigure[AC-GAN]{
\label{fid:acgan}
\includegraphics[height=0.21\textheight]{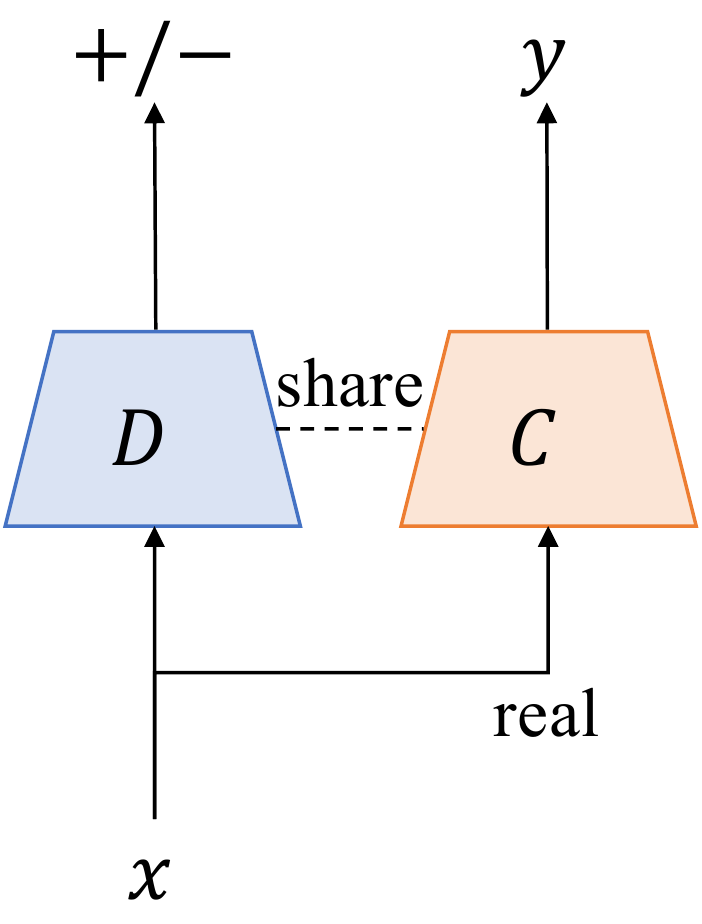}}
\subfigure[TAC-GAN]{
\label{fig:tacgan}
\includegraphics[height=0.21\textheight]{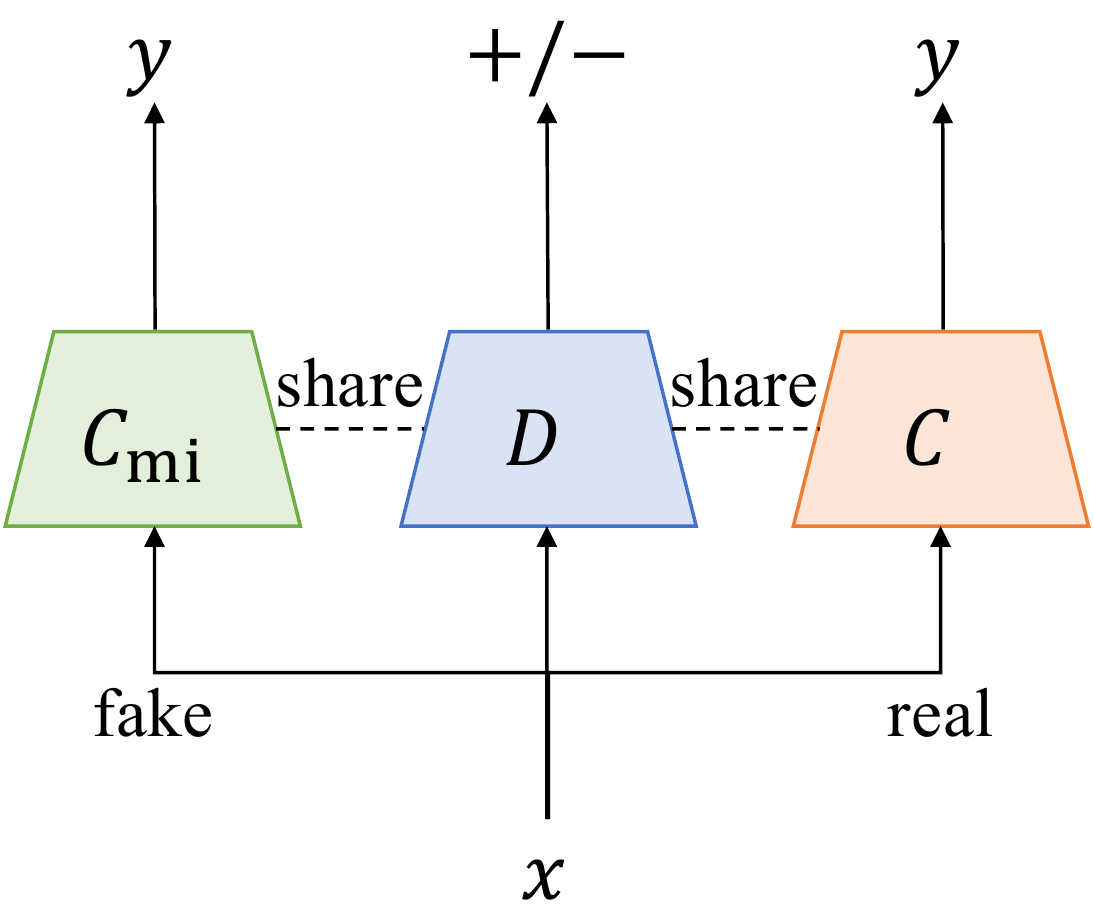}}
\subfigure[ADC-GAN]{
\label{fig:adcgan}
\includegraphics[height=0.21\textheight]{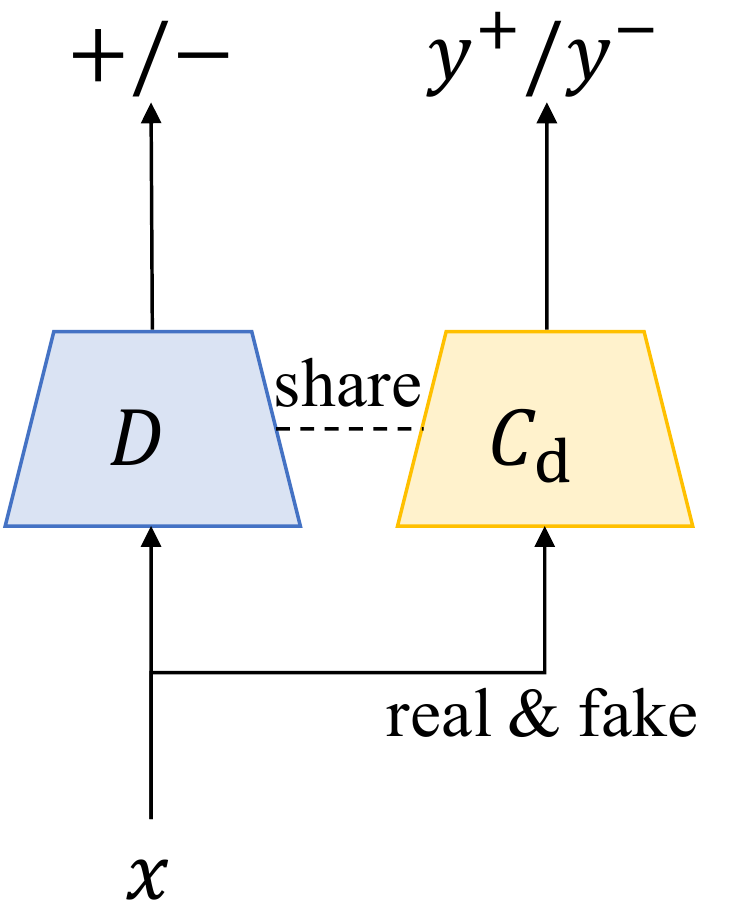}}
\caption{Illustration of discriminators/classifiers of existing cGANs (PD-GAN~\cite{miyato2018cgans}, AC-GAN~\cite{pmlr-v70-odena17a}, and TAC-GAN~\cite{NEURIPS2019_4ea06fbc}) and ADC-GAN. The symbol $+/-$ indicates the GAN labels (real or fake) and $y$ is the class-label of data $x$. ADC-GAN is different from PD-GAN by explicitly predicting the label and is different from AC-GAN and TAC-GAN in that the classifier $C_\mathrm{d}$ also distinguishes real from generated, like the discriminator.}
\label{fig:cgans}
\end{center}
\end{figure*}

The goal of conditional generative modeling is to faithfully learn the joint distribution of real data and labels regardless of the shape of the joint distribution (whether there is overlap between the distributions of different classes).
We first note that the reason why AC-GAN fails to learn the target joint distribution (\cref{thm:acgan-g}) originates from that the optimal classifier $C^*(y|x)=\frac{p(x,y)}{p(x)}$ (\cref{pro:acgan-c}) is agnostic to the density of the generated (marginal or joint) distribution ($q(x)$ or $q(x,y)$).
As a result, the classifier cannot provide the discrepancy between the target distribution and the generated distribution, resulting in a biased learning objective of the generator.
Recall that the optimal discriminator $D^*(x)=\frac{p(x)}{p(x)+q(x)}$ is aware of the real data distribution as well as the generated data distribution~\cite{NIPS2014_5ca3e9b1}, and can therefore provide the discrepancy between the real and generated data distributions $\frac{p(x)}{q(x)}=\frac{D^*(x)}{1-D^*(x)}$ for faithful generative modeling of the generator.
Intuitively, the distribution-aware ability on both real and generated data is caused by the fact that the discriminator distinguishes between the real and generated data with different labels (real or fake).
Motivated by this understanding, we propose to make the classifier capable of classifying the the real and generated data with different class-labels, establishing a discriminative classifier $C_\mathrm{d}:\mathcal{X}\to\mathcal{Y}^+\cup\mathcal{Y}^-$ ($\mathcal{Y}^+$ for real data and $\mathcal{Y}^-$ for generated data) that recognizes the label of the real and generated samples discriminatively.
The generator is encouraged to produce classifiable real data rather than classifiable fake data.
Mathematically, the objective functions for the discriminator, the discriminative classifier, and the generator of ADC-GAN are defined as:
\begin{IEEEeqnarray}{rCl}\label{eq:adcgan}
\max_{D,C_\mathrm{d}} V(G,D) + \lambda\cdot(\mathbb{E}_{x,y\sim P_{X,Y}}&[&\log C_\mathrm{d}(y^+|x)] \nonumber\\+\mathbb{E}_{x,y\sim Q_{X,Y}}&[&\log C_\mathrm{d}(y^-|x)]), \\
\min_{G} V(G,D) - \lambda\cdot(\mathbb{E}_{x,y\sim Q_{X,Y}}&[&\log C_\mathrm{d}(y^+|x)] \nonumber\\- \mathbb{E}_{x,y\sim Q_{X,Y}}&[&\log C_\mathrm{d}(y^-|x)]),
\end{IEEEeqnarray}
where $C_\mathrm{d}(y^+|x)=\frac{\exp(\varphi^+(y) \cdot \phi(x))}{\sum_{\bar{y}}\exp(\varphi^+(\bar{y}) \cdot \phi(x))+\sum_{\bar{y}}\exp(\varphi^-(\bar{y}) \cdot \phi(x))}$ (resp. $C_\mathrm{d}(y^-|x)=\frac{\exp(\varphi^-(y) \cdot \phi(x))}{\sum_{\bar{y}}\exp(\varphi^+(\bar{y}) \cdot \phi(x))+\sum_{\bar{y}}\exp(\varphi^-(\bar{y}) \cdot \phi(x))}$) indicates the probability that a data $x$ is classified as the label $y$ and real (resp. fake) simultaneously by the discriminative classifier.
Here, $\phi:\mathcal{X}\to\mathbb{R}^d$ is a feature extractor that is shared with the original discriminator in our implementation ($D=\sigma\circ\psi\circ\phi$ with a linear mapping $\psi:\mathbb{R}^d\to\mathbb{R}$ and a sigmoid function $\sigma:\mathbb{R}\to[0,1]$), and $\varphi^+:\mathcal{Y}\to\mathbb{R}^d$ and $\varphi^-:\mathcal{Y}\to\mathbb{R}^d$ capture learnable embeddings of labels responsible to the real and generated data, respectively.

\input{tables/objectives}

At the first glance, the objective function with the discriminative classifier for the generator seems to be redundant as maximization of
$\log C_\mathrm{d}(y^+|x)$ implicitly contains the goal of minimization of $\log C_\mathrm{d}(y^-|x)$.
However, we show below that the second term is indispensable for accurately learning the real joint data-label distribution.
Arguably, maximization of
$\log C_\mathrm{d}(y^+|x)$ forces the generator to produce only few label-consistent data, facilitating the fidelity but losing the diversity of the generated samples.
On the other hand, minimization of $\log C_\mathrm{d}(y^-|x)$ encourages the generator to not synthesis the typically label-consistent data, increasing the diversity but may degrade the fidelity of the generated samples.
In general, the two objectives together assist the generator in achieving its goal as we proved below.

\begin{restatable}{proposition}{adcganC}\label{pro:adcgan-c}
% \begin{proposition}
For fixed generator, the optimal discriminative classifier of ADC-GAN has the form of the following:
\begin{equation*}\label{eq:adcgan-c}
C_\mathrm{d}^*(y^+|x)=\frac{p(x,y)}{p(x)+q(x)}, C_\mathrm{d}^*(y^-|x)=\frac{q(x,y)}{p(x)+q(x)}.
\end{equation*}
% \end{proposition}
\end{restatable}

\cref{pro:adcgan-c} shows that the optimal discriminative classifier is aware of the densities of the real and generated joint distributions, therefore it is able to provide the discrepancy $\frac{p(x,y)}{q(x,y)}=\frac{C_\mathrm{d}^*(y^+|x)}{C_\mathrm{d}^*(y^-|x)}$ to optimize the generator.

\begin{restatable}{theorem}{adcganG}\label{thm:adcgan-g}
% \begin{theorem}\label{thm:adcgan-g}
Given the optimal discriminative classifier, at the equilibrium point, optimizing the classification task for the generator of ADC-GAN is equivalent to:
\begin{equation}\label{eq:adcgan-g}
\min_G \mathrm{KL}(Q_{X,Y}\|P_{X,Y}).
\end{equation}
% \end{theorem}
\end{restatable}

\cref{thm:adcgan-g} confirms that the discriminative classifier itself can guarantee the generator to restore the real joint distribution at the optimum.
In practice, we retain the discriminator to train the generator for better training stability and convergence.
The overall learning objective for the generator under the optimal discriminator and discriminative classfier is to minimize the JS divergence between the marginal data distributions and the reversed KL divergence bewteen the joint data-label distributions ($\min_G\mathrm{JS}(P_X\|Q_X) + \lambda\cdot\mathrm{KL}(Q_{X,Y}\|P_{X,Y})$).
Since the optimal solution set for generative modeling contains the optimal solution set for conditional generative modeling ($\arg\min_{G}\mathrm{JS}(P_X\|Q_X) \supseteq \arg\min_{G}\mathrm{KL}(Q_{X,Y}\|P_{X,Y})$), the guidance to the generator provided by discriminator and discriminative classifier are harmonious, which makes ADC-GAN robust to the value of the hyperparameter $\lambda$ and the selection of the GAN loss $V(G,D)$.
% Furthermore, the hyperparameter provides the flexibility to balance the weights between unconditional and conditional generative modeling.

\section{Analysis on Competing Methods}

In this section, we analyze the drawbacks of the two competing methods, TAC-GAN~\cite{NEURIPS2019_4ea06fbc} and PD-GAN~\cite{miyato2018cgans}, to show the superiority of ADC-GAN.
We also analyze AM-GAN~\cite{zhou2018activation} in \cref{sec:amgan}.
Before diving into the details, we show diagrams of the discriminator and classifier of these methods in \cref{fig:cgans} and summarize the theoretical learning objective for the generator under the optimal discriminator and classifier of these methods in \cref{tbl:objectives} for an overview.

\subsection{Competing Method: TAC-GAN}\label{sec:tacgan}

TAC-GAN~\cite{NEURIPS2019_4ea06fbc} addresses the low intra-class diversity problem of AC-GAN by eliminating the conditional entropy of the generated data distribution $H_Q(Y|X)$ by learning the generator with another classifier $C_\mathrm{mi}:\mathcal{X}\to\mathcal{Y}$, which is trained with the generated samples.
The objective functions for the discriminator, the twin classifiers, and the generator of TAC-GAN are defined as follows:
\begin{IEEEeqnarray}{rCl}\label{eq:tacgan}
\max_{D,C,C_\mathrm{mi}} V(G,D) + \lambda\cdot(\mathbb{E}_{x,y\sim P_{X,Y}}&[&\log C(y|x)] \nonumber \\ + \mathbb{E}_{x,y\sim Q_{X,Y}}&[&\log C_\mathrm{mi}(y|x)]), \\
\min_{G} V(G,D) - \lambda\cdot(\mathbb{E}_{x,y\sim Q_{X,Y}}&[&\log C(y|x)] \nonumber \\ - \mathbb{E}_{x,y\sim Q_{X,Y}}&[&\log C_\mathrm{mi}(y|x)]).
\end{IEEEeqnarray}

\begin{restatable}{theorem}{tacganG}\label{thm:tacgan-g}
% \begin{theorem}
Given the twin optimal classifiers, at the equilibrium point, optimizing the classification tasks for the generator of TAC-GAN is equivalent to:
\begin{equation}\label{eq:tacgan-g}
\min_G \mathrm{KL}(Q_{X,Y}\|P_{X,Y}) - \mathrm{KL}(Q_X\|P_X).
\end{equation}
% \end{theorem}
\end{restatable}

Our \cref{thm:tacgan-g} reveals that the learning objective of the generator of TAC-GAN, under the twin optimal classifiers, can be regarded as optimizing contradictory divergences, i.e.,~minimization between joint distributions but maximization between marginal distributions.
Although theoretically the JS divergence or others~\cite{NIPS2016_cedebb6e,pmlr-v70-arjovsky17a} introduced through the adversarial training between the discriminator and generator may remedy this issue, it is difficult to obtain the optimal discriminator and classifier in the practical optimization to ensure the elimination of the contradiction.
We argue that the training instability of TAC-GAN reported in the literature~\cite{kocaoglu2018causalgan,han2020unbiased} and found in our experiments (cf. \cref{fig:c100_fid_curve,fig:fid_curve}) can be explained by this analysis.

\subsection{Competing Method: PD-GAN}\label{sec:pdgan}

PD-GAN~\cite{miyato2018cgans} injects the conditional information into the projection discriminator $D_\mathrm{p}:\mathcal{X}\times\mathcal{Y}\to[0,1]$ via the inner-product between the embedding of the label and the representation of the data to calculate the joint discriminative score of the data-label pair.
In such a way, PD-GAN inherits the property of convergence point similar to the standard GAN such that it can avoid the low intra-class diversity problem of AC-GAN ideally.
Specifically, the objective functions for the projection discriminator and the generator of PD-GAN are defined as follows:
\begin{IEEEeqnarray}{rCl}\label{eq:pdgan}
\min_G\max_{D_\mathrm{p}} V(G,D_\mathrm{p}) &=& \mathbb{E}_{x,y\sim P_{X,Y}}[\log D_\mathrm{p}(x,y)] \nonumber \\ &+& \mathbb{E}_{x,y\sim Q_{X,Y}}[\log (1-D_\mathrm{p}(x,y))].\IEEEeqnarraynumspace
\end{IEEEeqnarray}
Based on this formulation, the optimal projection discriminator has the following form:
\begin{IEEEeqnarray}{rCl}
D_\mathrm{p}^*(x,y)&=&\frac{1}{1+\exp(-d^*(x,y))}=\frac{p(x,y)}{p(x,y)+q(x,y)} \nonumber\\
\Rightarrow  d^*(x,y)&=&\log\frac{p(x,y)}{q(x,y)}=\log\frac{p(x)}{q(x)}+\log\frac{p(y|x)}{q(y|x)},
\end{IEEEeqnarray}
where $p(y|x)=\frac{\exp(\varphi^+(y) \cdot \phi(x))}{\sum_{\bar{y}} \exp(\varphi^+(\bar{y}) \cdot \phi(x))}$ and $q(y|x)=\frac{\exp(\varphi^-(y) \cdot \phi(x))}{\sum_{\bar{y}} \exp(\varphi^-(\bar{y}) \cdot \phi(x))}$.
And PD-GAN accordingly defines:
\begin{IEEEeqnarray}{rCl}\label{eq:pd-gan}
&&r(x):=\log\frac{p(x)}{q(x)}:=\psi(\phi(x)), \nonumber\\
&&r(y|x):=\log\frac{p(y|x)}{q(y|x)}:=\underbrace{(\overbrace{\varphi^+(y)-\varphi^-(y)}^{\varphi(y)})\cdot \phi(x)}_{\hat{r}(y|x)} - \\ &&\underbrace{\log\sum_{\bar{y}\in\mathcal{Y}}\exp\left(\varphi^+(\bar{y}) \cdot \phi(x)\right)+\log\sum_{\bar{y}\in\mathcal{Y}}\exp\left(\varphi^-(\bar{y}) \cdot \phi(x)\right)}_{\textcircled{a}}.\nonumber
\end{IEEEeqnarray}

However, PD-GAN actually ignores the partition term $\textcircled{a}$\footnote{PD-GAN discards $\textcircled{a}$ in implementing the projection discriminator based on the hypothesis that $\textcircled{a}$ can be merged into $r(x)$. However, $r(x)$ does not model any label information, which should be involved by $\textcircled{a}$. Therefore, it is unreasonable to do this.
} in Equation~\ref{eq:pd-gan} and heuristically constructs the logit of the projection discriminator in the form of:
\begin{equation}
d(x,y)=r(x)+\hat{r}(y|x)=\psi(\phi(x))+\varphi(y)\cdot\phi(x).
\end{equation}
Discarding the partition term would make PD-GAN no longer belong to probability models that are able to model the conditional probabilities $p(y|x)$ and $q(y|x)$, resulting in losing the complete dependencies between data and labels.
Particularly, for mismatched data-label pair $(x,y)$ with probabilities of $p(x,y)=0$ and $q(x,y)=0$, the projection discriminator $D_\mathrm{p}^*(x,y)=\frac{p(x,y)}{p(x,y)+q(x,y)}=\frac{0}{0}$ is undefined and thus unreliable.
Our ADC-GAN can penalize the mismatched data-label pair because $C_\mathrm{d}^*(y^+|x)=\frac{p(x,y)}{p(x)+q(x)}=\frac{0}{>0}=0$ ($p(x)+q(x)>0$ for valid data $x$).
Moreover, the optimal projection discriminator constructed according to the minimax GAN lacks theoretical guarantees on other GAN loss functions.
The proposed ADC-GAN can be flexibly applied to any version of the GAN loss as we do not require a specific form of the discriminator.

\section{Experiments}

\begin{figure*}[t]
\begin{center}
\subfigure[Real Data]{\label{fig:synthetic_data}
\includegraphics[width=0.22\textwidth]{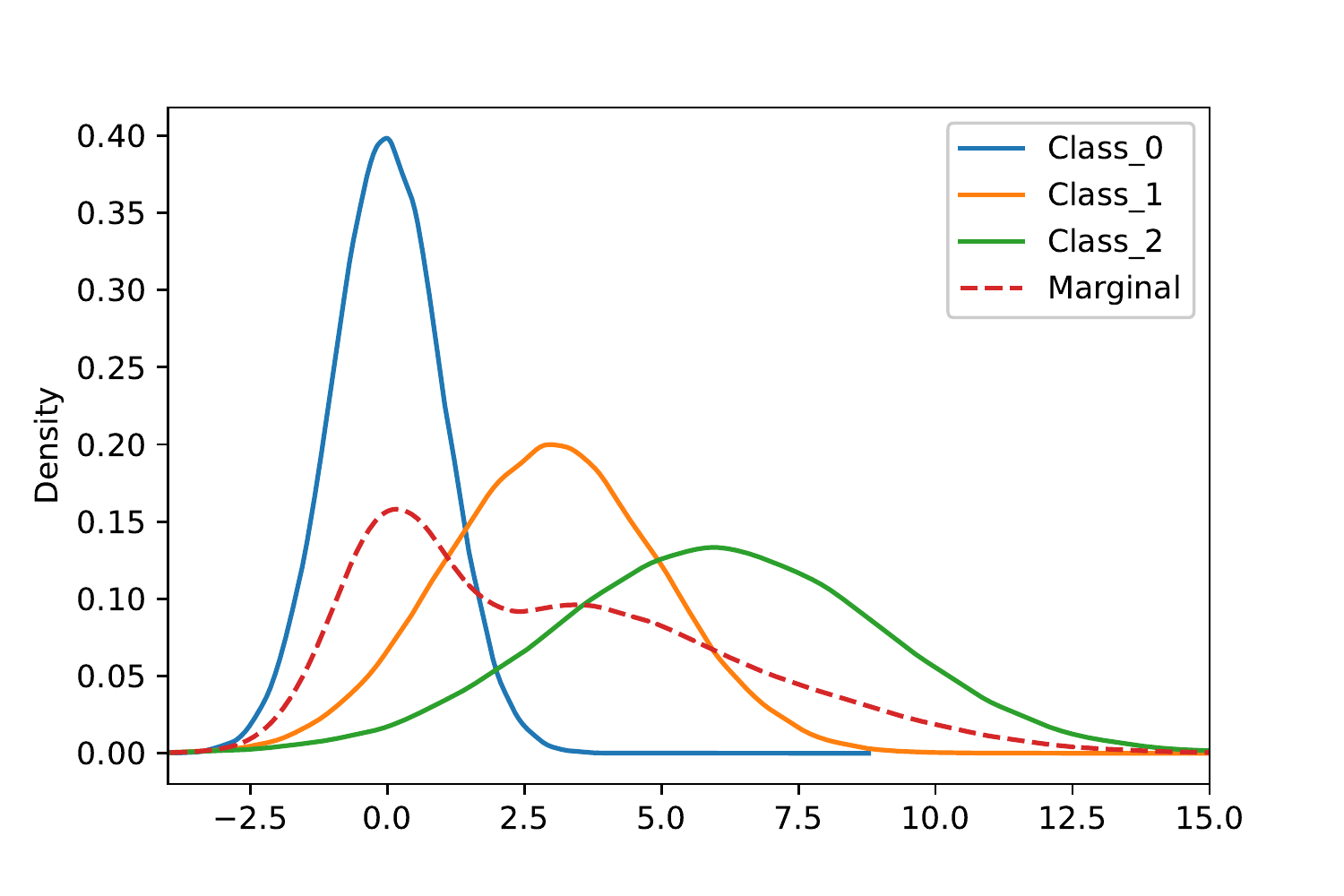}}
\subfigure[AC-GAN w/o $V(G,D)$]{\label{fig:synthetic_acgan_wo}
\includegraphics[width=0.22\textwidth]{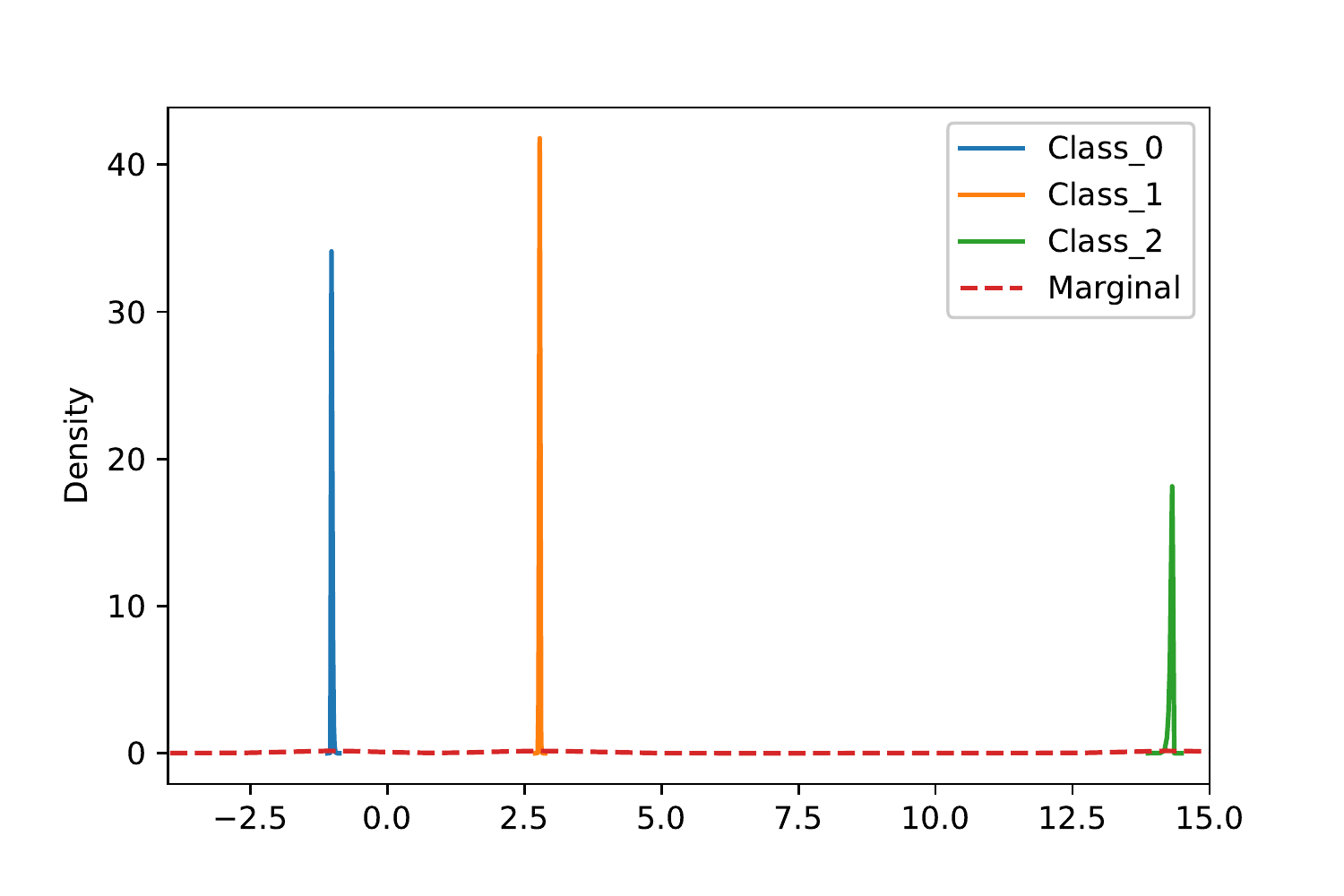}}
\subfigure[TAC-GAN w/o $V(G,D)$]{\label{fig:synthetic_tacgan_wo}
\includegraphics[width=0.22\textwidth]{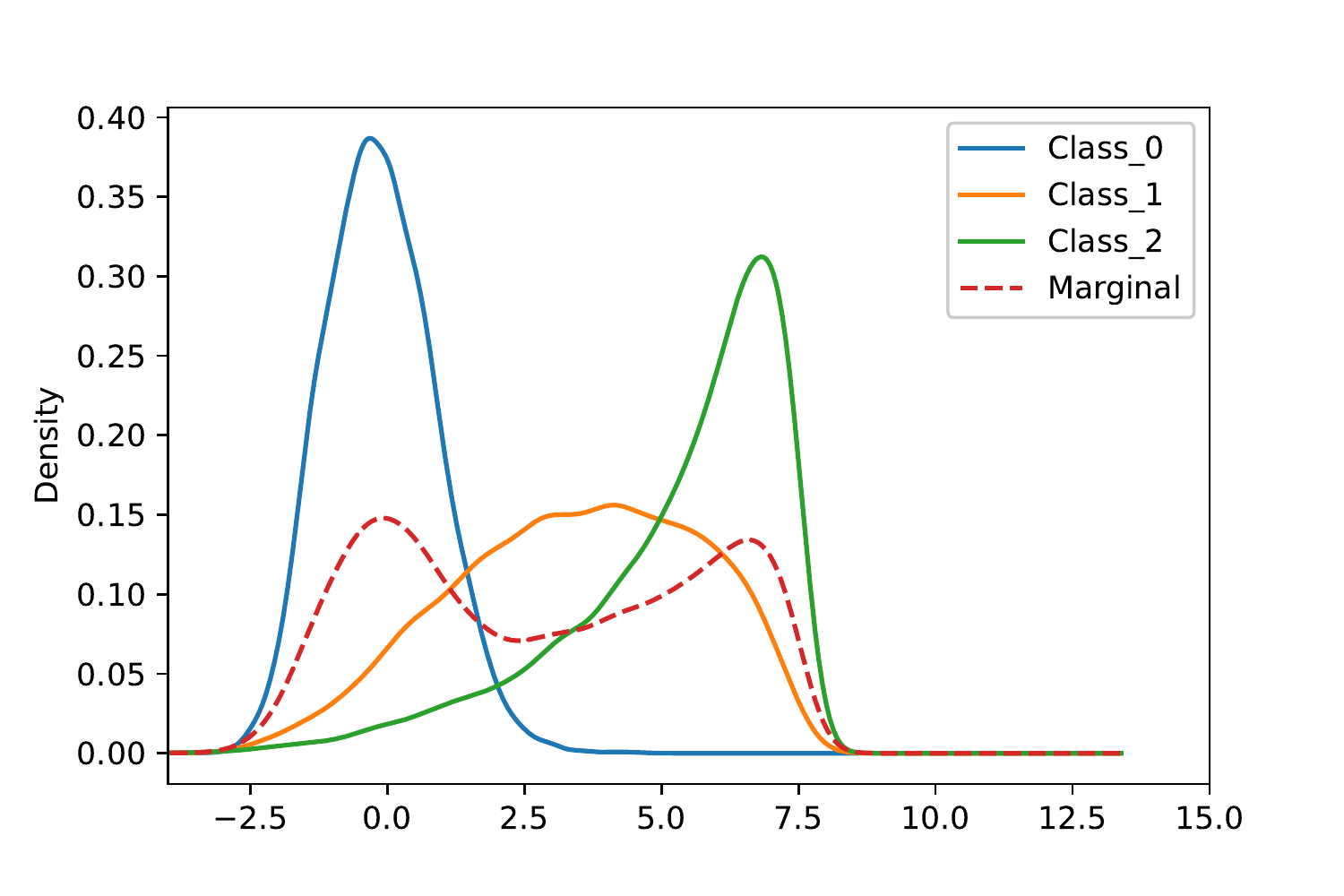}}
\subfigure[ADC-GAN w/o $V(G,D)$]{\label{fig:synthetic_adcgan_wo}
\includegraphics[width=0.22\textwidth]{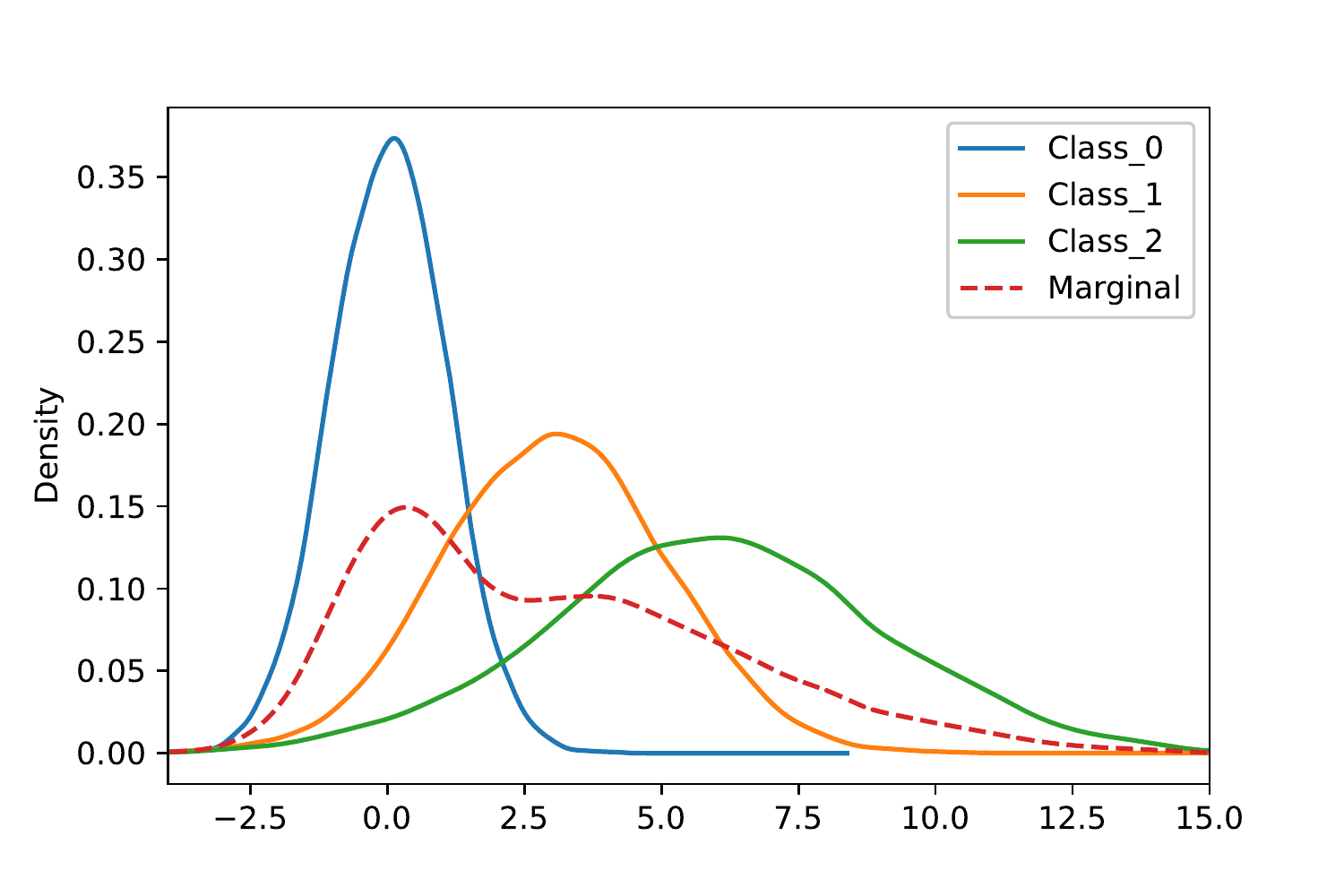}}
\subfigure[PD-GAN]{\label{fig:synthetic_pdgan}
\includegraphics[width=0.22\textwidth]{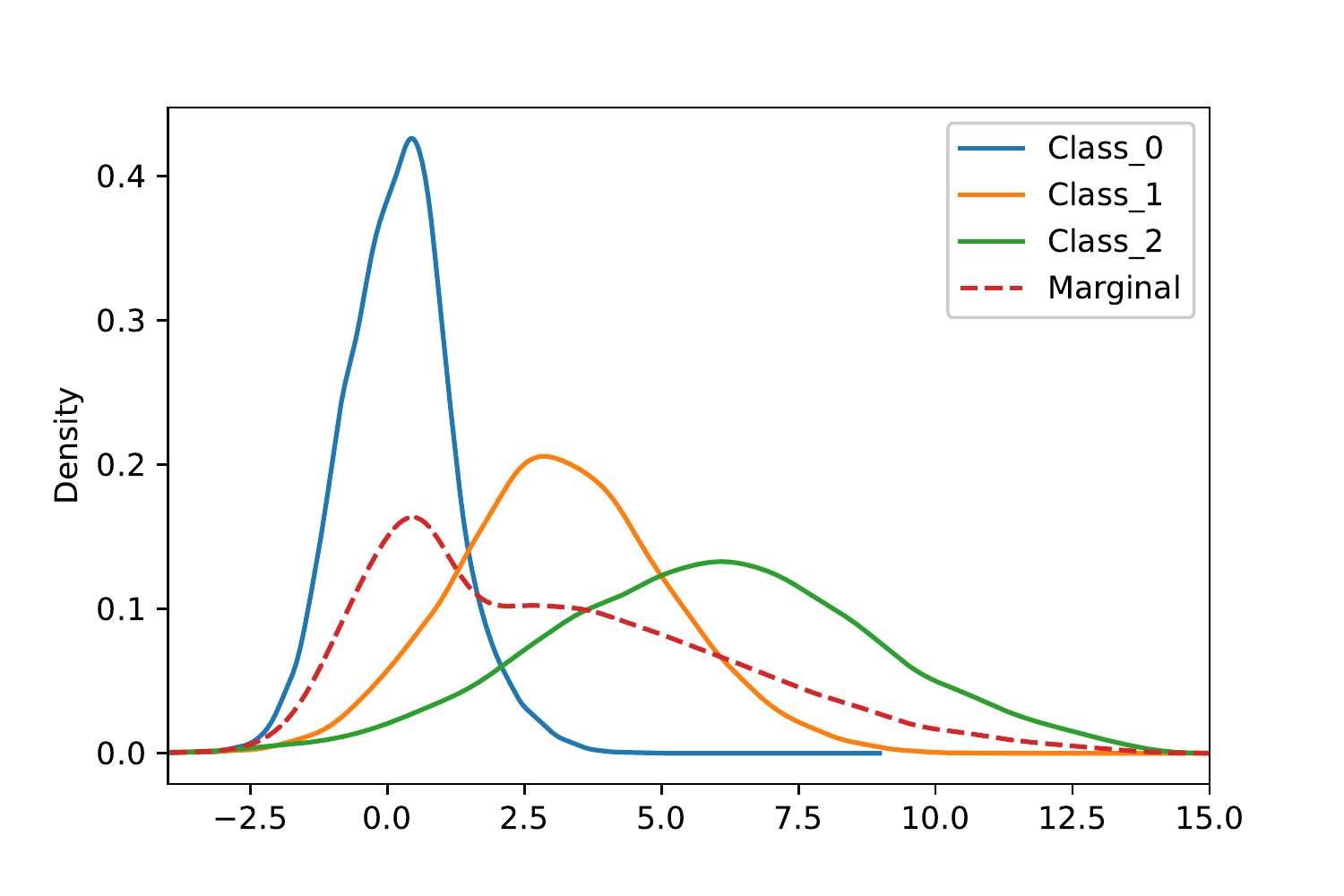}}
\subfigure[AC-GAN w/ $V(G,D)$]{\label{fig:synthetic_acgan}
\includegraphics[width=0.22\textwidth]{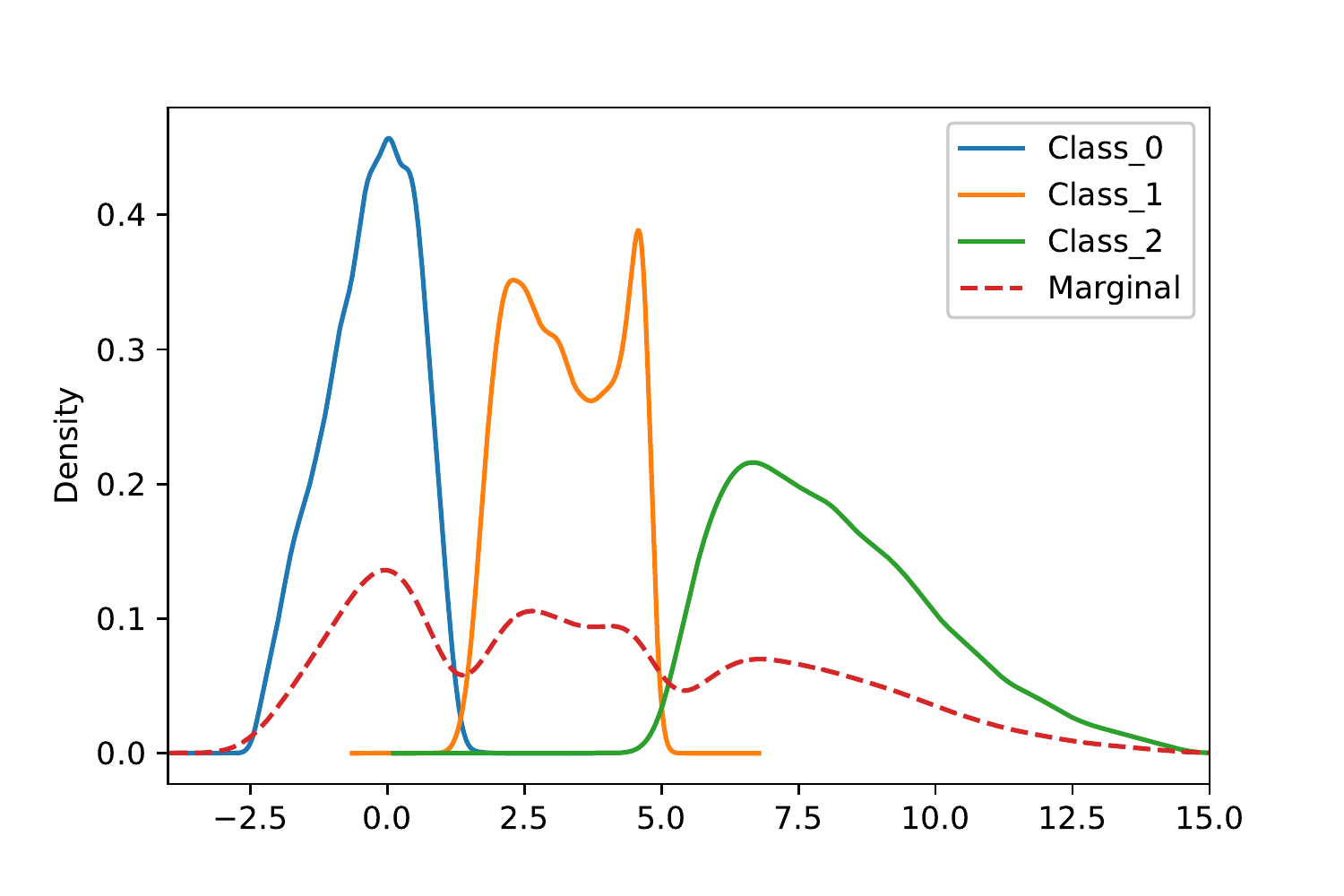}}
\subfigure[TAC-GAN w/ $V(G,D)$]{\label{fig:synthetic_tacgan}
\includegraphics[width=0.22\textwidth]{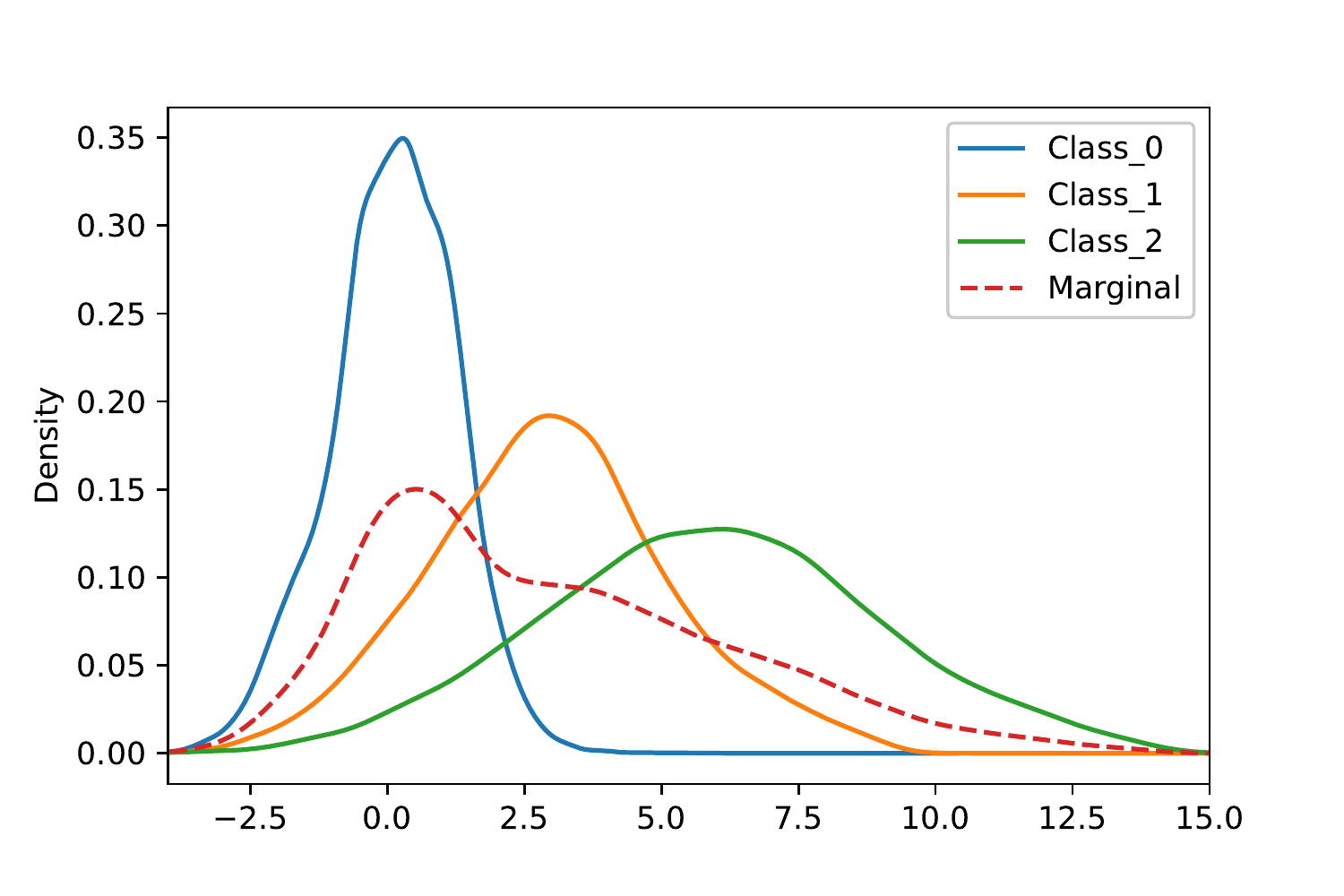}}
\subfigure[ADC-GAN w/ $V(G,D)$]{\label{fig:synthetic_adcgan}
\includegraphics[width=0.22\textwidth]{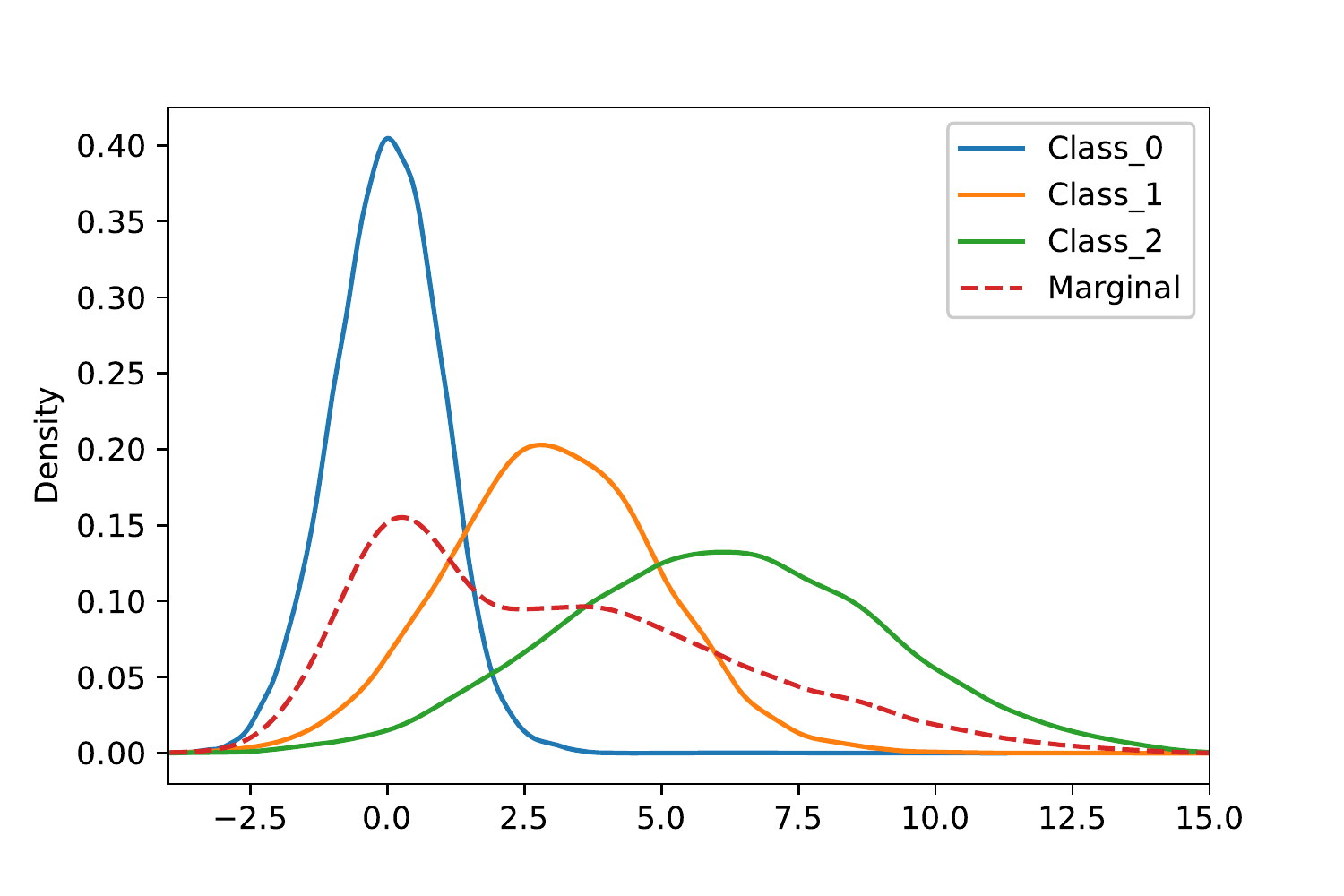}}
\caption{Qualitative comparison of distribution modeling results on the one-dimensional synthetic data.}
\label{fig:synthetic}
\end{center}
\end{figure*}

\input{tables/biggan}

\subsection{Synthetic Data}\label{sec:synthetic}

We first conduct experiments on a one-dimensional synthetic mixture of Gaussians, following the practices of~\cite{NEURIPS2019_4ea06fbc}, to qualitatively show the fidelity of distribution learning capability of ADC-GAN.
% In this experiment, both generator and discriminator are multi-layer perceptrons with non-linearity of Tanh.
As shown in \cref{fig:synthetic_data}, the real data distribution consists of three classes with non-negligible overlaps.
\cref{fig:synthetic_acgan_wo,fig:synthetic_tacgan_wo,fig:synthetic_adcgan_wo} show the learned distributions, which are estimated by kernel density estimation (KDE)~\cite{parzen1962estimation} on the generated data of AC-GAN, TAC-GAN, and ADC-GAN without the original GAN loss $V(G,D)$, respectively.
\cref{fig:synthetic_pdgan,fig:synthetic_acgan,fig:synthetic_tacgan,fig:synthetic_adcgan} show the KDE results of PD-GAN, AC-GAN, TAC-GAN, and ADC-GAN trained with the non-saturating GAN loss~\cite{NIPS2014_5ca3e9b1}, respectively.
AC-GAN tends to generate classifiable data so that it decreases the intra-class diversity.
Without the GAN loss $V(G,D)$, AC-GAN outputs nearly deterministic data for each class.
TAC-GAN without the GAN loss also cannot accurately capture the real data distribution, verifying the contradiction in \cref{thm:tacgan-g}.
Impressively, the proposed ADC-GAN faithfully restores the real data distribution even without the GAN loss, validating \cref{thm:adcgan-g} that the discriminative classfier alone can guide the generator to learn the real data distribution.

\subsection{Experiments based on BigGAN-PyTorch}

In this section, we conduct experiments on three common real-world datasets: CIFAR-10, CIFAR-100~\cite{krizhevsky2009learning}, and Tiny-ImageNet~\cite{le2015tiny} based on the BigGAN-PyTorch repository\footnote{\url{https://github.com/ajbrock/BigGAN-PyTorch}} with our extensions\footnote{\url{https://github.com/houliangict/adcgan}}.
The optimizer is Adam with learning rate of $2\times 10^{-4}$ on CIFAR-10/100 and $1\times 10^{-4}$ for the generator and $4\times 10^{-4}$ for the discriminator on Tiny-ImageNet.
We train all methods for $1000$ and $500$ epochs with batch size of $50$ and $100$ on CIFAR-10/100 and Tiny-ImageNet, respectively.
The discriminator/classifier are updated $4$ and $2$ times per generator update step on CIFAR-10/100 and Tiny-ImageNet, respectively.
We follow the practice of~\cite{miyato2018cgans,NEURIPS2019_4ea06fbc} to adopt the hinge loss~\cite{lim2017geometric,NIPS2017_6f1d0705} as the implementation of $V(G,D)$.
The coefficient hyperparameters of AC-GAN and AM-GAN~\cite{zhou2018activation} (cf. \cref{sec:amgan} for analysis) are set as $\lambda=0.2$ as it performs the best.
As for TAC-GAN and ADC-GAN, the coefficient hyperparameters are set as $\lambda=1.0$ on CIFAR-10/100 and $\lambda=0.5$ on Tiny-ImageNet.

\textbf{Image Generation.} We use the Fr\'{e}chet Inception Distance (FID)~\cite{NIPS2017_8a1d6947} and Intra-FID~\cite{miyato2018cgans} metrics to measure the overall and intra-class qualities of the generated images, respectively.
\cref{tbl:c10_c100_ti200} shows that ADC-GAN obtains the best FID and Intra-FID scores on all three datasets, indicating consistent superiority over previous cGANs in conditional image generation.

\textbf{Training Stability.} We also note that ADC-GAN yields the best training stability according to the FID training curves (cf. \cref{fig:c100_fid_curve,fig:fid_curve}).
Even without the discriminator, the training stability ADC-GAN (w/o D) still exceeds that of most competing methods.
AC-GAN diverges during training on all three datasets.
TAC-GAN also diverges on CIFAR-100 and Tiny-ImageNet and achieves a relatively stable FID training curve only on the simplest dataset, CIFAR-10.
We hence report the results of all methods using the best checkpoint.
These unstable FID training curves implicitly verify the drawback of existing classifier-based cGANs that optimize contradictory divergences.

\textbf{Different Coefficients.} To explicitly show the above issues, we set the objective function of classifier-based cGANs as $(1-\lambda')V(G,D)+\lambda'V_\mathrm{C}(G,C)$, where $V_\mathrm{C}(G,C)$ is the task between the generator and classifier.
As shown in \cref{fig:c100_lambda,fig:lambda}, ADC-GAN consistently gains superior FID scores across different coefficient hyperparameters even for $\lambda'=1.0$ (i.e.,~without the discriminator), showing strong robustness with respect to$\lambda'$, while AC-GAN and TAC-GAN perform substantially worse when $\lambda'$ becomes larger.

\begin{figure}[t]
\begin{center}
\subfigure[FID curve]{
\label{fig:c100_fid_curve}
\includegraphics[width=0.225\textwidth]{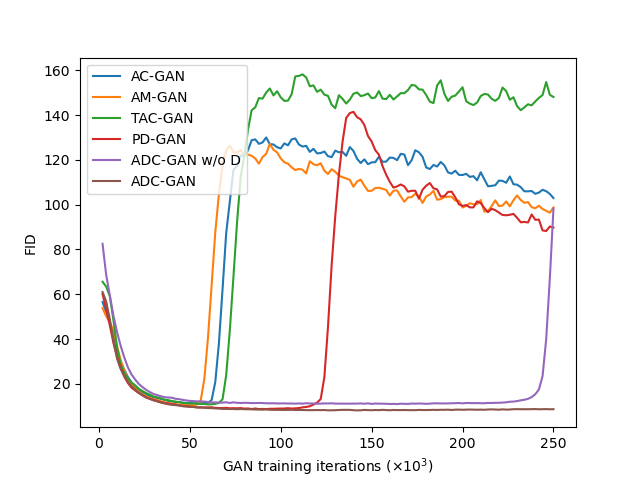}}
\subfigure[FID with different $\lambda'$]{
\label{fig:c100_lambda}
\includegraphics[width=0.225\textwidth]{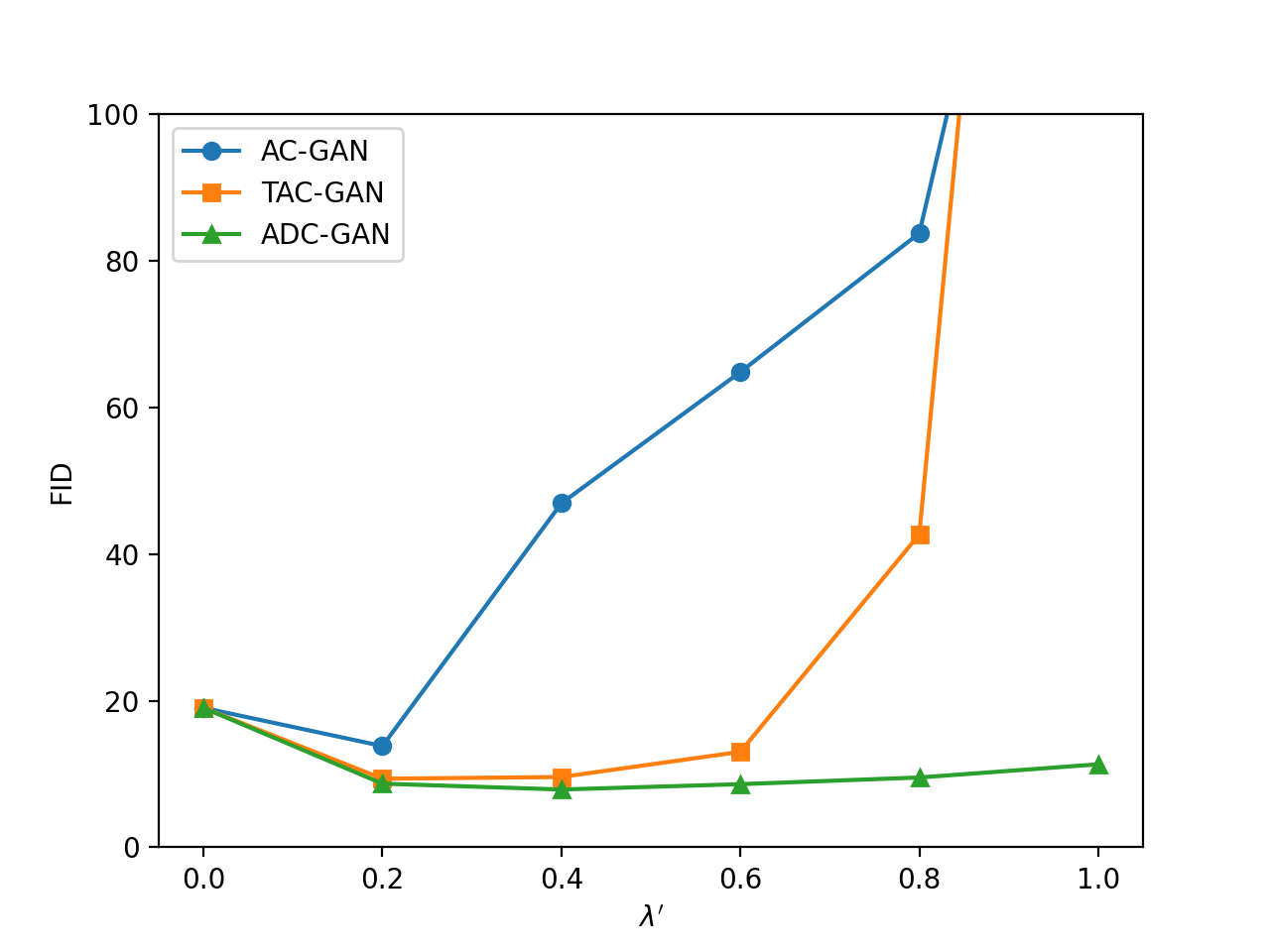}}
\caption{(a) FID curves during GAN training on CIFAR-100. (b) FID scores of classifier-based cGANs with different $\lambda'$ on CIFAR-100. The objective function in this experiment is $(1-\lambda')V(G,D)+\lambda'V_\mathrm{C}(G,C)$, where $V_\mathrm{C}(G,C)$ is the task between the generator and classifier.}
\label{fig:c100}
\end{center}
\end{figure}

\begin{figure}[t]
\begin{center}
\subfigure[PD-GAN]{
\label{fig:c10_tsne_pdgan}
\includegraphics[width=0.225\textwidth]{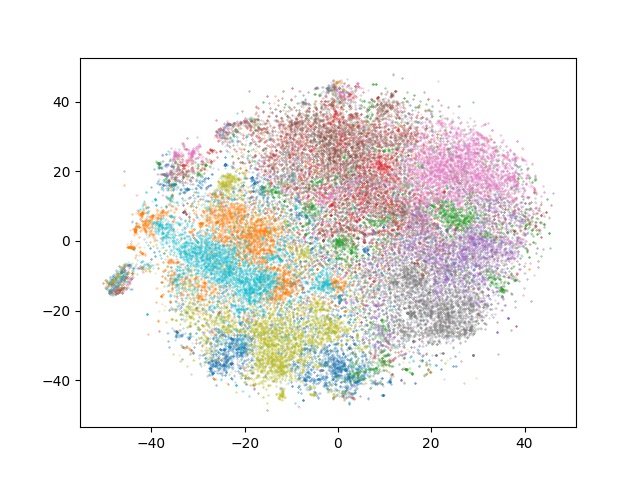}}
\subfigure[ADC-GAN]{
\label{fig:c10_tsne_adcgan}
\includegraphics[width=0.225\textwidth]{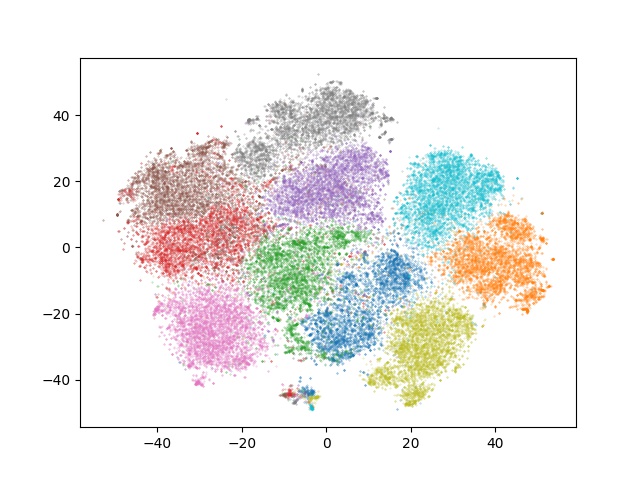}}
\caption{T-SNE visualization of CIFAR-10 validation data based on learned representations extracted from the penultimate layer in the discriminator/classifier $\phi(x)$. Different colors indicate different classes.}
\label{fig:tsne}
\end{center}
\end{figure}

\textbf{Data-to-Class Relations.} To investigate whether the model captures appropriate data-to-class relations, we conduct image classification experiments based on the learned representations of the discriminator/classifier $\phi(x)$.
Specifically, we first train a logistic regression classifier using the scikit-learn library with the training data and compute the classification accuracy of the validation data.
% Arguably, if the discriminator/classifier effectively learns the complete and accurate dependencies between data and labels, then the representations extracted from it will obtain higher accuracy in this experiment.
As reported in \cref{tbl:c10_c100_ti200}, ADC-GAN significantly outperforms competing methods on all datasets in terms of the Accuracy metrics.
The reason is that the discriminative classifier needs to recognize the labels of data while simultaneously distinguishing between real and fake data, which facilitates the robustness of the classifier in modeling data-to-class relations.
Notice that PD-GAN obtains the worst results.
By comparing the CIFAR-10 T-SNE~\cite{van2008visualizing} visualization results of PD-GAN and ADC-GAN in \cref{fig:tsne}, it is clear that PD-GAN does not have the ability to learn proper data-to-class relations as ADC-GAN does, reflecting the problem caused by the loss of partition terms in PD-GAN.

\subsection{Experiments based on PyTorch-StudioGAN}

\input{tables/imagenet}

In this section, we compare ADC-GAN with state-of-the-art cGANs using the PyTorch-StudioGAN repository\footnote{\url{https://github.com/POSTECH-CVLab/PyTorch-StudioGAN}}, of which evaluation protocols are different from that of the BigGAN-PyTorch repository that we used in \cref{tbl:c10_c100_ti200}.
Nonetheless, our comparison is fair because the methods in each experiment follows the same evaluation protocol.

\textbf{Image Generation on ImageNet.} We first conduct experiments on ImageNet ($128\times 128$) following the experimental settings of ReACGAN~\cite{kang2021rebooting}.
\cref{tbl:imagenet} reports the Inception Score (IS)~\cite{NIPS2016_8a3363ab} and FID results.
Our ADC-GAN is comparable with the state-of-the-art cGANs, BigGAN and ReACGAN~\cite{kang2021rebooting}, in the batch size of $256$ and $2048$, showing effectiveness on large-scale high-resolution image datasets.
Notice that, however, we only ran our ADC-GAN once with $\lambda=1$ in each of the two batch size settings, and did not make other attempts due to our limited computational resources.
We argue that the results of ADC-GAN can be improved by choosing an appropriate coefficient hyperparameter $\lambda$.

\textbf{Different GAN Losses.} We also investigate the robustness of ADC-GAN with respect to the GAN loss function $V(G,D)$ by adopting different versions.
% We compare with two SOTA cGANs, ContraGAN~\cite{NEURIPS2020_f490c742} and ReACGAN~\cite{kang2021rebooting}, which consider data-to-data relations as well as data-to-class relations.
\cref{tbl:gan_loss} report the qualitative results on CIFAR-100 (cf. \cref{tbl:gan_loss_all} in \cref{sec:more} for complete results).
Impressively, the proposed ADC-GAN achieves the best iFID (intra-FID), recall~\cite{NEURIPS2019_0234c510}, and coverage~\cite{pmlr-v119-naeem20a} scores across the non-saturation~\cite{NIPS2014_5ca3e9b1}, WGAN-GP~\cite{NIPS2017_892c3b1c}, and hinge~\cite{lim2017geometric} versions of the GAN loss.
The best iFID scores indicate the best conditional generative modeling performance, and the best recall and coverage results reflect the best (intra-class) diversity of the generated samples.

\input{tables/gan_loss}

\section{Related Work}

Efforts on developing cGANs~\cite{mirza2014conditional} can be divided into two steps.
The first is to study how to implement a conditional generator.
Methods in this category are concatenation~\cite{mirza2014conditional}, conditional batch normalization~\cite{NIPS2017_6fab6e3a}, and conditional convolution layers~\cite{sagong2019cgans}.
The second is to study how to train the conditional generator to produce label-dependent samples, which can be further divided into two categories, classifier-based and projection-based cGANs.

\textbf{Classifier-based cGANs.}
AC-GAN~\cite{pmlr-v70-odena17a} leveraged an auxiliary classifier to identify consistency between data and labels.
MH-GAN~\cite{Kavalerov_2021_WACV} improved AC-GAN by replacing the cross-entropy loss of the classifier with the multi-hinge loss.
AM-GAN~\cite{zhou2018activation} replaced the discriminator with a $K+1$-way classifier with an additional ``fake'' label.
Omni-GAN~\cite{zhou2020omni} combined the discriminator with the classifier to construct a $K+2$-dimensional multi-label classifier.
TAC-GAN~\cite{NEURIPS2019_4ea06fbc} corrected the biased learning objective of AC-GAN by introducing another classifier, which is the multi-class version of Anti-Labeler of CausalGAN~\cite{kocaoglu2018causalgan}.
UAC-GAN~\cite{han2020unbiased} improved the training stability of TAC-GAN with MINE~\cite{pmlr-v80-belghazi18a}.
ECGAN~\cite{chen2021a} provides a unified view of cGANs with and without classifiers.
Orthogonally to our work, ContraGAN~\cite{NEURIPS2020_f490c742} and ReACGAN~\cite{kang2021rebooting} modeled data-to-data relations as well as data-to-class relations using the conditional contrastive loss and the data-to-data cross-entropy loss, respectively.
However, they did not solve the low intra-class diversity problem of AC-GAN as they inherited the generator-agnostic classifier.
% Their ideas are orthogonal to ours, and we leave it as future work.

\textbf{Projection-based cGANs.}
PD-GAN~\cite{miyato2018cgans} injected the class information into the discriminator via label projection and achieved the state-of-the-art generation quality of natural images~\cite{brock2018large,wu2019logan,Zhang2020Consistency,zhao2021improved}.
P2GAN~\cite{Han_2021_ICCV} further improved PD-GAN by compensating the missed partition term in the objective function.

\textbf{Discriminative classifiers.}
\citet{pmlr-v139-watanabe21a} exploited the discriminative classifier for training GANs with any level of labeling but different from us with the objective function for the generator, which enables ADC-GAN to faithfully learn the target distribution.
SSGAN-LA~\cite{hou2021selfsupervised} presented the similar idea but different loss functions with ADC-GAN (multi-hinge v.s. cross-entropy) to tackle the degraded learning objective of self-supervised GANs, while ADC-GAN is for conditional GANs.
Moreover, our analysis of the degradation objective is more accurate and informative than that of SSGAN-LA.
% \cref{thm:acgan-g} reveals that the learning objectives of AC-GAN contains not only the conditional entropy that reduces the diversity of generated samples but also the contradictory maximization of divergence of marginal distributions, which are not stated in SSGAN-LA.

\vskip -0.05in
\section{Conclusion}

In this paper, we present a novel conditional generative adversarial network with an auxiliary discriminative classifier (ADC-GAN) to achieve faithful conditional generative modeling.
% The discriminative classifier can provide the discrepancy between the joint distribution of the real data and labels and that of the generated data and labels to the generator by discriminatively predicting the labels of the real and generated data.
% The generator can therefore faithfully learn the real joint data and label distribution at the Nash equilibrium.
We also discuss the differences between ADC-GAN with competing cGANs and analyze their potential issues and limitations.
Extensive experimental results validate the theoretical superiority of ADC-GAN compared with state-of-the-art classifier-based and projection-based cGANs.

\section*{Acknowledgements}

This work is funded by the National Natural Science Foundation of China under Grant Nos. 62102402, U21B2046, and National Key R\&D Program of China (2020AAA0105200). Huawei Shen is also supported by Beijing Academy of Artificial Intelligence (BAAI).


\bibliography{ref}
\bibliographystyle{icml2022}

%%%%%%%%%%%%%%%%%%%%%%%%%%%%%%%%%%%%%%%%%%%%%%%%%%%%%%%%%%%%%%%%%%%%%%%%%%%%%%%
%%%%%%%%%%%%%%%%%%%%%%%%%%%%%%%%%%%%%%%%%%%%%%%%%%%%%%%%%%%%%%%%%%%%%%%%%%%%%%%
% APPENDIX
%%%%%%%%%%%%%%%%%%%%%%%%%%%%%%%%%%%%%%%%%%%%%%%%%%%%%%%%%%%%%%%%%%%%%%%%%%%%%%%
%%%%%%%%%%%%%%%%%%%%%%%%%%%%%%%%%%%%%%%%%%%%%%%%%%%%%%%%%%%%%%%%%%%%%%%%%%%%%%%
\newpage
\appendix
\onecolumn
\section{Proofs}\label{sec:proofs}

\subsection{Proof of \cref{pro:acgan-c}}
\label{sec:proof-acgan-c}
\acganC*

\begin{proof}
\begin{IEEEeqnarray}{rCl}
&&\max_C\mathbb{E}_{x,y\sim P_{X,Y}}[\log C(y|x)] = \mathbb{E}_{x\sim P_X}\mathbb{E}_{y\sim P_{Y|X}}[\log C(y|x)]
\\
&\Rightarrow&\min_C\mathbb{E}_{x\sim P_X}\mathbb{E}_{y\sim P_{Y|X}}[-\log C(y|x)] = \mathbb{E}_{x\sim P_X}[H(p(y|x))+\mathrm{KL}(p(y|x)\|C(y|x))] \\
&\Rightarrow&C^*(y|x)=\arg\min_C\mathrm{KL}(p(y|x)\|C(y|x))=p(y|x)=\frac{p(x,y)}{p(x)}
\end{IEEEeqnarray}
\end{proof}

\subsection{Proof of \cref{thm:acgan-g}}
\label{sec:proof-acgan-g}
\acganG*

\begin{proof}
\begin{IEEEeqnarray}{rCl}
&&\max_G\mathbb{E}_{x,y\sim Q_{X,Y}}\left[\log C^*(y|x)\right]=\mathbb{E}_{x,y\sim Q_{X,Y}}\left[\log \frac{p(x,y)}{p(x)}\right]=\mathbb{E}_{x,y\sim Q_{X,Y}}\left[\log \frac{p(x,y)}{q(x,y)}\frac{q(x)}{p(x)}\frac{q(x,y)}{q(x)}\right]\\
&\Rightarrow&\min_G\mathbb{E}_{x,y\sim Q_{X,Y}}\left[\log\frac{q(x,y)}{p(x,y)}\right]-\mathbb{E}_{x\sim Q_{X}}\left[\log\frac{q(x)}{p(x)}\right]-\mathbb{E}_{x,y\sim Q_{X,Y}}\left[\log\frac{q(x,y)}{q(x)}\right] \\
&\Rightarrow&\min_G \mathrm{KL}(Q_{X,Y}\|P_{X,Y})-\mathrm{KL}(Q_{X}\|P_{X})+H_Q(Y|X)
\end{IEEEeqnarray}
\end{proof}

\subsection{Proof of \cref{pro:adcgan-c}}
\label{sec:proof-adcgan-c}
\adcganC*

\begin{proof}
\begin{equation}
\max_{C_\mathrm{d}}\mathbb{E}_{x,y\sim P_{X,Y}}[\log C_\mathrm{d}(y^+|x)] + \mathbb{E}_{x,y\sim Q_{X,Y}}[\log C_\mathrm{d}(y^-|x)]
\Rightarrow\max_{C_\mathrm{d}}\mathbb{E}_{x,y\sim P^m_{X,Y}}[\log C_\mathrm{d}(y|x)],
\end{equation}
with $p^m(x,y^+)=\frac{1}{2}p(x,y)$, $p^m(x,y^-)=\frac{1}{2}q(x,y)$, and $p^m(x)=\sum_y p^m(x,y)=\frac{1}{2}p(x)+\frac{1}{2}q(x)$.
\begin{IEEEeqnarray}{rCl}
&\Rightarrow&\max_{C_\mathrm{d}}\mathbb{E}_{x\sim P^m_{X}}\mathbb{E}_{y\sim P^m_{Y|X}}[\log C_\mathrm{d}(y|x)]\Rightarrow\min_{C_\mathrm{d}}\mathbb{E}_{x\sim P^m_{X}}\mathbb{E}_{y\sim P^m_{Y|X}}[-\log C_\mathrm{d}(y|x)]\\
&\Rightarrow&\min_{C_\mathrm{d}} \mathbb{E}_{x\sim P^m_X}[H(p^m(y|x))+\mathrm{KL}(p^m(y|x)\|C_\mathrm{d}(y|x))] \\
&\Rightarrow& C_\mathrm{d}^*(y|x)=\arg\min_{C_\mathrm{d}}\mathrm{KL}(p^m(y|x)\|C_\mathrm{d}(y|x))=p^m(y|x)=\frac{p^m(x,y)}{p^m(x)}
\end{IEEEeqnarray}
Therefore, the optimal discriminative classifier of ADC-GAN has the form of $C_\mathrm{d}^*(y^+|x)=\frac{p^m(x,y^+)}{p^m(x)}=\frac{p(x,y)}{p(x)+q(x)}$ and $C_\mathrm{d}^*(y^-|x)=\frac{p^m(x,y^-)}{p^m(x)}=\frac{q(x,y)}{p(x)+q(x)}$ that conclude the proof.

\end{proof}

\subsection{Proof of \cref{thm:adcgan-g}}
\label{sec:proof-adcgan-g}
\adcganG*
\begin{proof}
\begin{IEEEeqnarray}{rCl}
&&\max_G \mathbb{E}_{x,y\sim Q_{X,Y}}\left[\log C_\mathrm{d}^*(y^+|x)\right] - \mathbb{E}_{x,y\sim Q_{X,Y}}\left[\log C_\mathrm{d}^*(y^-|x)\right] \\
&\Rightarrow&\max_G \mathbb{E}_{x,y\sim Q_{X,Y}}\left[\log \frac{p(x,y)}{p(x)+q(x)}\right] - \mathbb{E}_{x,y\sim Q_{X,Y}}\left[\log \frac{q(x,y)}{p(x)+q(x)}\right] \\
% &\Rightarrow&\min_G - \mathbb{E}_{x,y\sim Q_{X,Y}}\left[\log \frac{p(x,y)}{p(x)+q(x)}\right] + \mathbb{E}_{x,y\sim Q_{X,Y}}\left[\log \frac{q(x,y)}{p(x)+q(x)}\right] \\
&\Rightarrow&\min_G\mathbb{E}_{x,y\sim Q_{X,Y}}\left[\log \frac{q(x,y)}{p(x,y)}\right]\Rightarrow \min_G \mathrm{KL}(Q_{X,Y}\|P_{X,Y})
\end{IEEEeqnarray}
\end{proof}

\subsection{Proof of \cref{thm:tacgan-g}}
\label{sec:proof-tacgan-g}

\begin{proposition}\label{pro:tacgan-c}
For fixed generator, the twin optimal classifiers of TAC-GAN have the following forms:
\begin{equation}\label{eq:tacgan-c}
C^*(y|x)=\frac{p(x,y)}{p(x)},
C_\mathrm{mi}^*(y|x)=\frac{q(x,y)}{q(x)}.
\end{equation}
\end{proposition}

\begin{proof}

The proof is similar to that of \cref{pro:acgan-c} in \cref{sec:proof-acgan-c} by considering $C$ and $C_\mathrm{mi}$ as two independent classifiers with respect to distribution $P$ and $Q$, respectively.
\end{proof}

\tacganG*

\begin{proof}
\begin{IEEEeqnarray}{rCl}
&&\max_G\mathbb{E}_{x,y\sim Q_{X,Y}}[\log C^*(y|x)] - \mathbb{E}_{x,y\sim Q_{X,Y}}[\log C_\mathrm{mi}^*(y|x)] \\
&\Rightarrow &\max_G\mathbb{E}_{x,y\sim Q_{X,Y}}\left[\log \frac{p(x,y)}{p(x)}\right] - \mathbb{E}_{x,y\sim Q_{X,Y}}\left[\log \frac{q(x,y)}{q(x)}\right] \\
&\Rightarrow &\max_G\mathbb{E}_{x,y\sim Q_{X,Y}}\left[\log \frac{p(x,y)}{q(x,y)}\right] - \mathbb{E}_{x\sim Q_{X}}\left[\log \frac{p(x)}{q(x)}\right]\\
&\Rightarrow &\min_G\mathrm{KL}(Q_{X,Y}\|P_{X,Y}) - \mathrm{KL}(Q_{X}\|P_{X})
\end{IEEEeqnarray}
\end{proof}

\section{Analysis on the Original AC-GAN}\label{sec:acgan_full}

In this section, we show that original AC-GAN whose auxiliary classifier is trained with both real and generated samples still suffers from the same issue as we proved in \cref{thm:acgan-g}.
Formally, the full objective function of the original AC-GAN is formulated as the following:
\begin{IEEEeqnarray}{rCl}\label{eq:acgan_full}
\max_{D,C} V(G,D) & + & \lambda\cdot\left(\mathbb{E}_{x,y\sim P_{X,Y}}[\log C(y|x)]+\mathbb{E}_{x,y\sim Q_{X,Y}}[\log C(y|x)]\right), \\
\min_{G} V(G,D) & - & \lambda\cdot\left(\mathbb{E}_{x,y\sim Q_{X,Y}}[\log C(y|x)]\right).
\end{IEEEeqnarray}

The objective function for training the classifier can be rewritten as:
\begin{equation}
\max_C \mathbb{E}_{x,y\sim P_{X,Y}}[\log C(y|x)]+\mathbb{E}_{x,y\sim Q_{X,Y}}[\log C(y|x)]\Rightarrow \max_C \mathbb{E}_{x,y\sim P_{X,Y}^m}[\log C(y|x)],
\end{equation}
with $p^m(x,y)=\frac{1}{2}(p(x,y)+q(x,y))$ and $p^m(x)=\sum_y p^m(x,y)=\frac{1}{2}(p(x)+q(x))$.
And we can obtain the optimal classifier according to the following:
\begin{IEEEeqnarray}{rCl}
&&\max_C \mathbb{E}_{x,y\sim P_{X,Y}^m}[\log C(y|x)] \Rightarrow \min_C \mathbb{E}_{x\sim P_{X}^m,y\sim P^m_{Y|X}}[-\log C(y|x)] \\
&\Rightarrow& \min_C \mathbb{E}_{x\sim P_X^m}[H(p^m(y|x)) + \mathrm{KL}(p^m(y|x)\|C(y|x))] \\
&\Rightarrow& C^*(y|x)=p^m(y|x)=\frac{p(x,y)+q(x,y)}{p(x)+q(x)}.
\end{IEEEeqnarray}

Suppose that the conditional generator learns the joint distribution of real data and labels, i.e.,~$q(x,y)=p(x,y)$ and $q(x)=p(x)$, the optimal classifier $C^*(y|x)=\frac{p(x,y)+q(x,y)}{p(x)+q(x)}=\frac{p(x,y)}{p(x)}$ also provide the objective stated in \cref{thm:acgan-g} for the generator, which contains the conditional entropy of the generated samples $H_Q(Y|X)$ that reduces the intra-class diversity of the generated samples. 
In other words, the original classifier does not allow the generator to remain on the desired distribution because it still provides momentum to update the generator, resulting in a biased learning objective for the generator in the original version of AC-GAN.
The essential reason is that the classifier of the original AC-GAN is incapable of distinguishing the real data from the generated data.
Therefore, the classifier of the original AC-GAN cannot provide the difference between the real and generated joint distributions to optimize the generator.

\section{Analysis on AM-GAN}\label{sec:amgan}

AM-GAN~\cite{zhou2018activation} optimizes the following objectives with an label-extended discriminator $D_+:\mathcal{X}\to\mathcal{Y}\cup\{0\}$:

\begin{IEEEeqnarray}{rCl}\label{eq:amgan}
\max_{D_+} \mathbb{E}_{x,y\sim P_{X,Y}}&[&\log D_+(y|x)]+\mathbb{E}_{x,y\sim Q_{X,Y}}[\log D_+(0|x)], \\
\min_{G} \mathbb{E}_{x,y\sim Q_{X,Y}}&[&\log D_+(y|x)].
\end{IEEEeqnarray}
The objective function for training the discriminator $D_+$ can be rewritten as:
\begin{equation}
\max_{D_+} \mathbb{E}_{x,y\sim P_{X,Y}}[\log D_+(y|x)]+\mathbb{E}_{x,y\sim Q_{X,Y}}[\log D_+(0|x)]\Rightarrow \max_{D_+}\mathbb{E}_{x,y\sim P^m_{X,Y}}[\log D_+(y|x)],
\end{equation}
where $p^m(x,y)=\frac{1}{2}p(x,y), \forall y\in\mathcal{Y}$, $p^m(x,0)=\frac{1}{2}q(x)$, and $p^m(x)=\sum_y p^m(x,y)=\frac{1}{2}(p(x)+q(x))$.
Then we have:
\begin{IEEEeqnarray}{rCl}
&&\max_{D_+}\mathbb{E}_{x,y\sim P^m_{X,Y}}[\log D_+(y|x)]
\Rightarrow\min_{D_+}\mathbb{E}_{x\sim P^m_X,y\sim P^m_{Y|X}}[-\log D_+(y|x)]\\ 
&\Rightarrow&\min_{D_+}\mathbb{E}_{x\sim P^m_X}[H(p^m(y|x))+\mathrm{KL}(p^m(y|x)\|D^+(y|x))]
\Rightarrow D_+^*(y|x)=p^m(y|x)=\frac{p(x,y)}{p(x)+q(x)},\forall y\in\mathcal{Y}.
\end{IEEEeqnarray}
Under the optimal discriminator $D_+^*$, the generator of AM-GAN can be regarded as optimizing the following:
\begin{IEEEeqnarray}{rCl}
&&\max_{G}\mathbb{E}_{x,y\sim Q_{X,Y}}[\log D_+^*(y|x)]
\Rightarrow\max_{G}\mathbb{E}_{x,y\sim Q_{X,Y}}\left[\log \frac{p(x,y)}{p(x)+q(x)}\right] \\
&\Rightarrow&\min_{G}\mathbb{E}_{x,y\sim Q_{X,Y}}\left[\log \frac{q(x,y)}{p(x,y)}\frac{p(x)+q(x)}{q(x,y)}\right]=\mathbb{E}_{x,y\sim Q_{X,Y}}\left[\log \frac{q(x,y)}{p(x,y)}+\log\frac{p(x)+q(x)}{2}-\log q(x,y)+\log 2\right]\IEEEeqnarraynumspace\\
&\geq &\min_G\mathbb{E}_{x,y\sim Q_{X,Y}}\left[\log \frac{q(x,y)}{p(x,y)}+\frac{1}{2}\log p(x)+\frac{1}{2}\log q(x)-\log q(x,y)+\log 2\right] \\
&\Rightarrow&\min_G\mathbb{E}_{x,y\sim Q_{X,Y}}\left[\log \frac{q(x,y)}{p(x,y)}-\frac{1}{2}\log \frac{q(x)}{p(x)}-\log\frac{q(x,y)}{q(x)}+\log 2\right] \\
&\Rightarrow&\min_G\mathrm{KL}(Q_{X,Y}\|P_{X,Y})-\frac{1}{2}\mathrm{KL}(Q_X\|P_X)+H_Q(Y|X)+\log 2.
\end{IEEEeqnarray}
In summary, AM-GAN with the original discriminator remained (compared in our experiments) can be considered to be minimizing an upper bound of $\mathrm{JS}(Q_X\|P_X)+\mathrm{KL}(Q_{X,Y}\|P_{X,Y})-\frac{1}{2}\mathrm{KL}(Q_X\|P_X)+H_Q(Y|X)+\log 2$.

\clearpage

\section{More Results}\label{sec:more}

\begin{figure*}[htbp]
\begin{center}
\subfigure[FID curves on CIFAR-10]{
\label{fig:c10_fid_curve}
\includegraphics[width=0.48\textwidth]{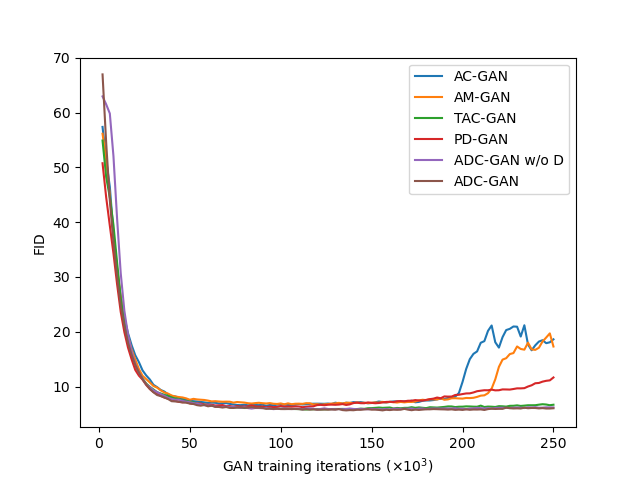}}
\subfigure[FID curves on Tiny-ImageNet]{
\label{fig:ti200_fid_curve}
\includegraphics[width=0.48\textwidth]{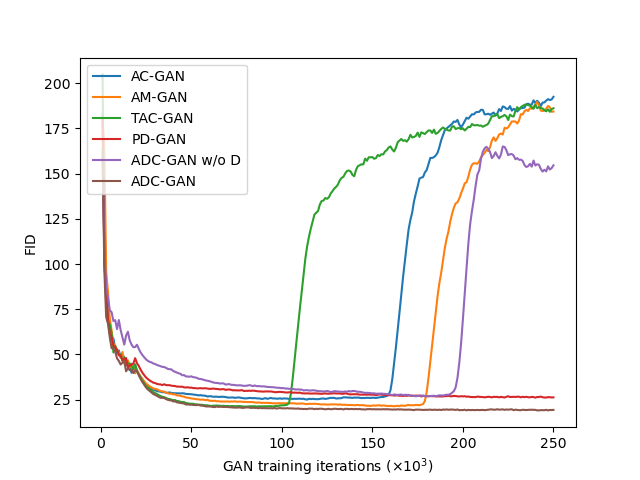}}
\caption{FID curves during GAN training on CIFAR-10 and Tiny-ImageNet, respectively.}
\label{fig:fid_curve}
\end{center}
\end{figure*}

\begin{figure*}[htbp]
\begin{center}
\subfigure[FID with different $\lambda'$ on CIFAR-10]{
\label{fig:c10_lambda}
\includegraphics[width=0.48\textwidth]{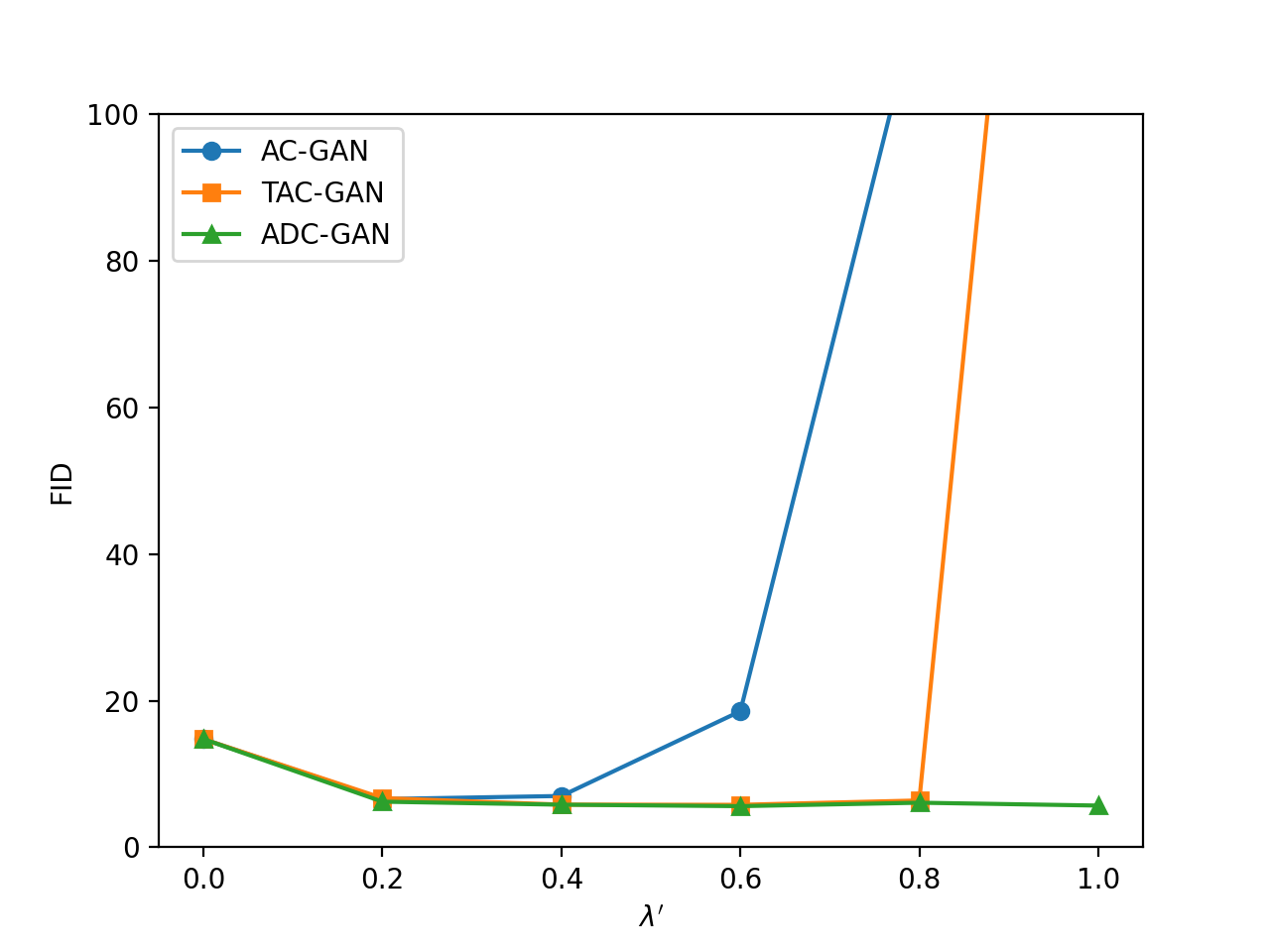}}
\subfigure[FID with different $\lambda'$ on Tiny-ImageNet]{
\label{fig:ti200_lambda}
\includegraphics[width=0.48\textwidth]{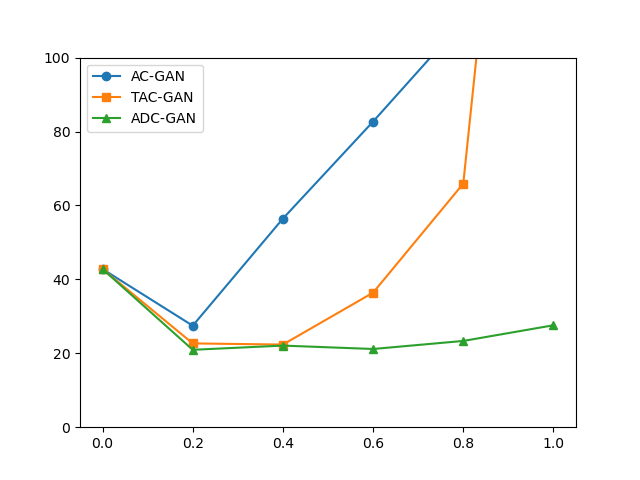}}
\caption{FID comparisons of classifier-based cGANs with different coefficient hyperparameters $\lambda'$ on CIFAR-10 and Tiny-ImageNet, respectively. The objective function in this experiment is $(1-\lambda')V(G,D)+\lambda'V_\mathrm{C}(G,C)$, where $V_\mathrm{C}(G,C)$ is the task between the generator and classifier.}
\label{fig:lambda}
\end{center}
\end{figure*}

\input{tables/gan_loss_all}
%%%%%%%%%%%%%%%%%%%%%%%%%%%%%%%%%%%%%%%%%%%%%%%%%%%%%%%%%%%%%%%%%%%%%%%%%%%%%%%
%%%%%%%%%%%%%%%%%%%%%%%%%%%%%%%%%%%%%%%%%%%%%%%%%%%%%%%%%%%%%%%%%%%%%%%%%%%%%%%

\end{document}

%% file: tables/objectives.tex
\begin{table*}[tbp]
\caption{Theoretical learning objective for the generator of competing methods under the optimal discriminator and classifier.}
\label{tbl:objectives}
% \vskip 0.1in
\begin{center}
\begin{sc}
\begin{tabular}{cl}
\toprule
Method & Theoretical Learning Objective for the Generator \\
\midrule
AC-GAN~\cite{pmlr-v70-odena17a} & $\min_G \mathrm{JS}(P_X\|Q_X)+\lambda(\mathrm{KL}(Q_{X,Y}\|P_{X,Y})-\mathrm{KL}(Q_X\|P_X)+H_Q(Y|X))$ \\
% AM-GAN~\citep{zhou2018activation} & $\min_G D(P_X\|Q_X)+\mathrm{KL}(Q_{X,Y}\|P_{X,Y})\mathbin{\color{white}{-}}{\color{white}\mathrm{KL}(Q_X\|P_X)}+H_Q(Y|X)$ \\
TAC-GAN~\cite{NEURIPS2019_4ea06fbc} & $\min_G \mathrm{JS}(P_X\|Q_X)+\lambda(\mathrm{KL}(Q_{X,Y}\|P_{X,Y})-\mathrm{KL}(Q_X\|P_X))$ \\
ADC-GAN (ours) & $\min_G \mathrm{JS}(P_X\|Q_X)+\lambda(\mathrm{KL}(Q_{X,Y}\|P_{X,Y}))$ \\
PD-GAN~\cite{miyato2018cgans} & $\min_G \mathrm{JS}(Q_{X,Y}\|P_{X,Y})$ \\
\bottomrule
\end{tabular}
\end{sc}
\end{center}
% \vskip -0.1in
\end{table*}

%% file: tables/biggan.tex
\begin{table*}[tbp]
\caption{FID and Intra-FID and Accuracy (\%) comparisons on CIFAR-10, CIFAR-100, and Tiny-ImageNet, respectively.}
\label{tbl:c10_c100_ti200}
% \vskip 0.1in
\begin{center}
\begin{sc}
\begin{tabular}{ccccccc}
\toprule
Datasets & Metrics & PD-GAN & AC-GAN & AM-GAN & TAC-GAN & ADC-GAN \\
\midrule
\multirow{3}*{CIFAR-10}
& FID~($\downarrow$) & $6.23$ & $6.50$ & $6.81$ & $5.83$ & $\mathbf{5.66}$ \\
& Intra-FID~($\downarrow$) & $48.90$ & $57.67$ & $69.31$ & $56.67$ & $\mathbf{40.45}$ \\
& Accuracy~($\uparrow$) & $66.22$ & $84.69$ & $83.63$ & $88.27$ & $\mathbf{89.51}$ \\
\midrule
\multirow{3}*{CIFAR-100}
& FID~($\downarrow$) & $8.70$ & $11.24$ & $10.42$ & $10.38$ & $\mathbf{8.12}$ \\
& Intra-FID~($\downarrow$) & $51.15$ & $83.06$ & $78.11$ & $79.59$ & $\mathbf{49.24}$ \\
& Accuracy~($\uparrow$) & $37.89$ & $55.26$ & $55.77$ & $60.03$ & $\mathbf{64.24}$ \\
\midrule
\multirow{3}*{Tiny-ImageNet}
& FID~($\downarrow$) & $26.10$ & $25.02$ & $21.34$ & $21.12$ & $\mathbf{19.02}$ \\
& Intra-FID~($\downarrow$) & $66.23$ & $99.04$ & $90.56$ & $95.48$ & $\mathbf{63.05}$ \\
& Accuracy~($\uparrow$) & $27.79$ & $44.59$ & $44.67$ & $44.44$ & $\mathbf{48.89}$ \\
\bottomrule
\end{tabular}
\end{sc}
\end{center}
% \vskip -0.1in
\end{table*}

%% file: tables/imagenet.tex
\begin{table}[tbp]
\caption{FID and IS comparisons on ImageNet ($128\times 128$). B.S. means the batch size and Iters. means the training iterations. Results of BigGAN and ReACGAN are copied from the ReACGAN paper~\cite{kang2021rebooting}.}
\label{tbl:imagenet}
% \vskip 0.1in
\begin{center}
\begin{sc}
\begin{tabular}{ccccc}
\toprule
B.S. & Iters. & Methods & IS ($\uparrow$) & FID ($\downarrow$) \\
\midrule
\multirow{3}*{$256$} & \multirow{3}*{$500$k} & BigGAN & $43.97$ & $16.36$ \\
& & ReACGAN & $\mathbf{68.27}$ & $13.98$ \\
& & ADC-GAN & $66.96$ & $\mathbf{11.65}$ \\
\midrule
\multirow{4}*{$2048$} & \multirow{3}*{$200$k} & BigGAN & $99.71$ & $\mathbf{7.89}$ \\
& & ReACGAN & $92.74$ & $8.23$ \\
& & ADC-GAN & $97.47$ & $9.46$ \\
\cmidrule{2-5}
& $500$k & ADC-GAN & $\mathbf{108.10}$ & $8.02$ \\
\bottomrule
\end{tabular}
\end{sc}
\end{center}
% \vskip -0.1in
\end{table}

%% file: tables/gan_loss.tex
\begin{table*}[tbp]
\caption{IS, FID, iFID, Precision, Recall, Density, and Coverage comparisons with state-of-the-art methods under different GAN loss functions on CIFAR-100, respectively. The best results are bold and the second best are underlined.}
\label{tbl:gan_loss}
\vskip -0.1in
\begin{center}
\begin{small}
\begin{sc}
\begin{tabular}{ccccccccc}
\toprule
GAN Loss & Methods & IS $\uparrow$ & FID $\downarrow$ & iFID $\downarrow$ & Precision $\uparrow$ & Recall $\uparrow$ & Density $\uparrow$ & Coverage $\uparrow$ \\
\midrule
\multirow{6}*{Non-saturation}
& PD-GAN & $11.48$ & $\underline{11.59}$ & $\underline{105.38}$ & $0.7337$ & $\underline{0.6804}$ & $\underline{0.8646}$ & $\underline{0.8513}$ \\
& AC-GAN & $7.98$ & $49.46$ & $207.56$ & $0.7322$ & $0.0793$ & $0.6225$ & $0.4112$ \\
& TAC-GAN & $11.34$ & $14.47$ & $131.90$ & $\underline{0.7429}$ & $0.6077$ & $0.8324$ & $0.7887$ \\
& ADC-GAN & $\mathbf{11.88}$ & $\mathbf{11.07}$ & $\mathbf{104.21}$ & $0.7379$ & $\mathbf{0.6972}$ & $0.8521$ & $\mathbf{0.8609}$ \\
& ContraGAN & $11.15$ & $13.54$ & $146.86$ & $0.7390$	& $0.6155$ & $0.8481$ & $0.7729$ \\
& ReACGAN & $\underline{11.79}$ & $13.72$ & $125.21$ & $\mathbf{0.7541}$ & $0.5861$ & $\mathbf{0.8695}$ & $0.8005$ \\
\midrule
\multirow{6}*{W-GP}
& PD-GAN & $5.66$ & $69.48$ & $-$ & $0.5976$ & $0.1603$ & $0.4310$ & $0.2649$ \\
& AC-GAN & $10.97$ & $19.30$ & $148.40$ & $0.6880$ & $0.5444$ & $0.6770$ & $0.7242$ \\
& TAC-GAN & $\mathbf{11.04}$ & $\underline{15.56}$ & $\underline{121.23}$ & $\underline{0.7023}$ & $\underline{0.6474}$ & $\underline{0.7048}$ & $\underline{0.7535}$ \\
& ADC-GAN & $\underline{11.01}$ & $\mathbf{14.02}$ & $\mathbf{101.14}$ & $\mathbf{0.7058}$ & $\mathbf{0.6804}$ & $\mathbf{0.7549}$ & $\mathbf{0.7956}$ \\
& ContraGAN & $6.72$ & $49.77$ & $147.22$ & $0.6498$ &	$0.2834$ & $0.5827$ & $0.3549$ \\
& ReACGAN & $6.67$ & $47.74$ & $150.7$ & $0.6188$ & $0.3104$ & $0.4806$ & $0.3396$ \\
\midrule
\multirow{6}*{Hinge}
& PD-GAN & $11.76$ & $\underline{10.96}$ & $\underline{108.08}$ & $0.7436$ & $\underline{0.6812}$ & $0.8790$ & $\underline{0.8609}$ \\
& AC-GAN & $11.66$ & $21.65$ & $168.87$ & $\mathbf{0.7577}$ & $0.3649$ & $0.8297$ & $0.7225$ \\
& TAC-GAN & $\mathbf{12.07}$ & $12.56$ & $134.75$ & $\underline{0.7572}$ & $0.6020$ & $\underline{0.8957}$ & $0.8400$ \\
& ADC-GAN & $\underline{11.82}$ & $\mathbf{10.73}$ & $\mathbf{103.78}$ & $0.7387$ & $\mathbf{0.7023}$ & $0.8721$ & $\mathbf{0.8707}$ \\
& ContraGAN & $10.08$ & $13.22$ & $128.50$ & $0.7372$ & $0.6251$ & $0.8356$ & $0.7790$ \\
& ReACGAN & $11.80$ & $12.52$ & $140.47$ & $0.7510$ & $0.5982$ & $\mathbf{0.9300}$ & $0.8327$ \\
\bottomrule
\end{tabular}
\end{sc}
\end{small}
\end{center}
\vskip -0.15in
\end{table*}

%% file: tables/gan_loss_all.tex
\begin{table*}[t]
\caption{IS, FID, iFID, Precision, Recall, Density, and Coverage comparisons of competing methods under different GAN loss functions on CIFAR-10 and CIFAR-100, respectively. The best results are bold and the second best are underlined.}
\label{tbl:gan_loss_all}
% \vskip 0.1in
\begin{center}
\begin{small}
\begin{sc}
\begin{tabular}{ccccccccc}
\toprule
CIFAR-10 & Methods & IS $\uparrow$ & FID $\downarrow$ & iFID $\downarrow$ & Precision $\uparrow$ & Recall $\uparrow$ & Density $\uparrow$ & Coverage $\uparrow$ \\
\midrule
\multirow{6}*{Non-saturation}
& PD-GAN & $9.68$ & $8.93$ & $81.30$ & $0.7581$ & $\underline{0.6718}$ & $\mathbf{1.0622}$ & $\underline{0.9208}$ \\
& AC-GAN & $\underline{9.74}$ & $9.21$ & $87.76$ & $0.7592$ & $0.6484$ & $1.0491$ & $0.9147$ \\
& TAC-GAN & $9.61$ & $9.31$ & $\underline{81.04}$ & $0.7349$ & $0.6717$ & $0.9575$ & $0.8990$ \\
& ADC-GAN & $\mathbf{9.87}$ & $\mathbf{8.47}$ & $\mathbf{77.69}$ & $0.7497$ & $\mathbf{0.6912}$ & $0.9968$ & $0.9202$  \\
& ContraGAN & $9.60$ & $8.87$ & $120.45$ & $\underline{0.7598}$ & $0.6595$ & $1.0025$ & $0.9061$ \\
& ReACGAN & $9.69$ & $\underline{8.51}$ & $113.23$ & $\mathbf{0.7648}$ & $0.6594$ & $\underline{1.0532}$ & $\mathbf{0.9242}$ \\
\midrule
\multirow{6}*{Least square}
& PD-GAN & $\mathbf{9.99}$ & $\underline{8.72}$ & $\underline{80.11}$ & $0.7525$ & $\underline{0.6771}$ & $\underline{1.0395}$ & $\underline{0.9182}$ \\
& AC-GAN & $5.01$ & $81.93$ & $176.24$ & $0.7389$ & $0.0037$ & $0.7484$ & $0.2129$ \\
& TAC-GAN & $9.41$ & $10.67$ & $80.92$ & $0.7386$ & $0.6520$ & $0.9159$ & $0.8657$ \\
& ADC-GAN & $9.89$ & $\mathbf{8.61}$ & $\mathbf{75.86}$ & $0.7405$ & $\mathbf{0.6919}$ & $0.9944$ & $\mathbf{0.9223}$ \\
& ContraGAN & $9.10$ & $12.93$ & $135.75$ & $\underline{0.7661}$ & $0.5761$ & $1.0236$ & $0.8262$ \\
& ReACGAN & $\underline{9.80}$ & $9.52$ & $125.83$ & $\mathbf{0.7772}$ & $0.5988$ & $\mathbf{1.1008}$ & $0.9138$ \\
\midrule
\multirow{6}*{W-GP}
& PD-GAN & $5.27$ & $75.24$ & $104.15$ & $0.5569$ & $0.2132$ & $0.3678$ & $0.2141$ \\
& AC-GAN & $8.88$ & $14.77$ & $88.02$ & $\mathbf{0.7015}$ & $0.6477$ & $0.7421$ & $0.7798$ \\
& TAC-GAN & $\underline{8.93}$ & $\underline{13.26}$ & $\underline{76.93}$ & $0.6847$ & $\underline{0.6705}$ & $\underline{0.7454}$ & $\underline{0.8127}$ \\
& ADC-GAN & $\mathbf{9.49}$ & $\mathbf{11.25}$ & $\mathbf{74.98}$ & $\underline{0.6996}$ & $\mathbf{0.7019}$ & $\mathbf{0.8182}$ & $\mathbf{0.8517}$ \\
& ContraGAN & $6.38$ & $51.43$ & $137.17$ & $0.5640$ & $0.3995$ & $0.4040$ & $0.2931$ \\
& ReACGAN & $6.60$ & $44.62$ & $117.25$ & $0.5813$ & $0.4333$ & $0.4559$ & $0.3287$ \\
\midrule
\multirow{6}*{Hinge}
& PD-GAN & $9.79$ & $\underline{8.45}$ & $79.40$ & $0.7464$ & $\underline{0.6853}$ & $1.0083$ & $0.9158$ \\
& AC-GAN & $\mathbf{9.96}$ & $8.97$ & $88.40$ & $\mathbf{0.7681}$ & $0.6523$ & $\underline{1.0250}$ & $\underline{0.9168}$ \\
& TAC-GAN & $9.78$ & $8.80$ & $81.30$ & $0.7446$ & $0.6749$ & $1.0026$ & $0.9103$ \\
& ADC-GAN & $9.63$ & $\mathbf{8.42}$ & $\mathbf{75.50}$ & $0.7447$ & $\mathbf{0.6882}$ & $0.9854$ & $\mathbf{0.9193}$ \\
& ContraGAN & $9.63$ & $8.89$ & $85.39$ & $0.7582$ & $0.6538$ & $\mathbf{1.0411}$ & $0.9098$ \\
& ReACGAN & $\underline{9.83}$ & $8.84$ & $\underline{78.07}$ & $\underline{0.7623}$ & $0.6675$ & $1.0003$ & $0.9158$ \\
\midrule
CIFAR-100 & Methods & IS $\uparrow$ & FID $\downarrow$ & iFID $\downarrow$ & Precision $\uparrow$ & Recall $\uparrow$ & Density $\uparrow$ & Coverage $\uparrow$ \\
\midrule
\multirow{6}*{Non-saturation}
& PD-GAN & $11.48$ & $\underline{11.59}$ & $\underline{105.38}$ & $0.7337$ & $\underline{0.6804}$ & $\underline{0.8646}$ & $\underline{0.8513}$ \\
& AC-GAN & $7.98$ & $49.46$ & $207.56$ & $0.7322$ & $0.0793$ & $0.6225$ & $0.4112$ \\
& TAC-GAN & $11.34$ & $14.47$ & $131.90$ & $\underline{0.7429}$ & $0.6077$ & $0.8324$ & $0.7887$ \\
& ADC-GAN & $\mathbf{11.88}$ & $\mathbf{11.07}$ & $\mathbf{104.21}$ & $0.7379$ & $\mathbf{0.6972}$ & $0.8521$ & $\mathbf{0.8609}$ \\
& ContraGAN & $11.15$ & $13.54$ & $146.86$ & $0.7390$	& $0.6155$ & $0.8481$ & $0.7729$ \\
& ReACGAN & $\underline{11.79}$ & $13.72$ & $125.21$ & $\mathbf{0.7541}$ & $0.5861$ & $\mathbf{0.8695}$ & $0.8005$ \\
\midrule
\multirow{6}*{Least square}
& PD-GAN & $11.32$ & $\underline{12.19}$ & $\mathbf{101.92}$ & $0.7263$ & $\underline{0.6903}$ & $0.8318$ & $\underline{0.8471}$ \\
& AC-GAN & $4.93$ & $87.70$ & $252.85$ & $0.7087$ & $0.0007$ & $0.5836$ & $0.2220$ \\
& TAC-GAN & $7.27$ & $49.08$ & $162.58$ & $0.7427$ & $0.2114$ & $0.7210$ & $0.4438$ \\
& ADC-GAN & $11.56$ & $\mathbf{11.85}$ & $\underline{103.06}$ & $0.7334$ & $\mathbf{0.6949}$ & $0.8145$ & $\mathbf{0.8526}$ \\
& ContraGAN & $\underline{12.59}$ & $15.62$ & $122.71$ & $\mathbf{0.7866}$	& $0.4642$ & $\underline{1.0109}$ & $0.7863$ \\
& ReACGAN & $\mathbf{12.90}$ & $15.09$ & $164.93$ & $\underline{0.7827}$ & $0.4672$ & $\mathbf{1.0454}$ & $0.8282$ \\
\midrule
\multirow{6}*{W-GP}
& PD-GAN & $5.66$ & $69.48$ & $-$ & $0.5976$ & $0.1603$ & $0.4310$ & $0.2649$ \\
& AC-GAN & $10.97$ & $19.30$ & $148.40$ & $0.6880$ & $0.5444$ & $0.6770$ & $0.7242$ \\
& TAC-GAN & $\mathbf{11.04}$ & $\underline{15.56}$ & $\underline{121.23}$ & $\underline{0.7023}$ & $\underline{0.6474}$ & $\underline{0.7048}$ & $\underline{0.7535}$ \\
& ADC-GAN & $\underline{11.01}$ & $\mathbf{14.02}$ & $\mathbf{101.14}$ & $\mathbf{0.7058}$ & $\mathbf{0.6804}$ & $\mathbf{0.7549}$ & $\mathbf{0.7956}$ \\
& ContraGAN & $6.72$ & $49.77$ & $147.22$ & $0.6498$ &	$0.2834$ & $0.5827$ & $0.3549$ \\
& ReACGAN & $6.67$ & $47.74$ & $150.7$ & $0.6188$ & $0.3104$ & $0.4806$ & $0.3396$ \\
\midrule
\multirow{6}*{Hinge}
& PD-GAN & $11.76$ & $\underline{10.96}$ & $\underline{108.08}$ & $0.7436$ & $\underline{0.6812}$ & $0.8790$ & $\underline{0.8609}$ \\
& AC-GAN & $11.66$ & $21.65$ & $168.87$ & $\mathbf{0.7577}$ & $0.3649$ & $0.8297$ & $0.7225$ \\
& TAC-GAN & $\mathbf{12.07}$ & $12.56$ & $134.75$ & $\underline{0.7572}$ & $0.6020$ & $\underline{0.8957}$ & $0.8400$ \\
& ADC-GAN & $\underline{11.82}$ & $\mathbf{10.73}$ & $\mathbf{103.78}$ & $0.7387$ & $\mathbf{0.7023}$ & $0.8721$ & $\mathbf{0.8707}$ \\
& ContraGAN & $10.08$ & $13.22$ & $128.50$ & $0.7372$ & $0.6251$ & $0.8356$ & $0.7790$ \\
& ReACGAN & $11.80$ & $12.52$ & $140.47$ & $0.7510$ & $0.5982$ & $\mathbf{0.9300}$ & $0.8327$ \\
\bottomrule
\end{tabular}
\end{sc}
\end{small}
\end{center}
% \vskip -0.1in
\end{table*}